\documentclass{article} \usepackage{iclr2026_conference_preprint,times}
\usepackage[T1]{fontenc}
\usepackage[utf8]{inputenc}
\usepackage{hyperref}
\usepackage{url}
\usepackage{graphicx}
\usepackage{wrapfig}
\usepackage{multirow}
\usepackage{amsmath,amssymb,amsfonts,amsthm,mathtools}
\usepackage{bm}
\usepackage{subcaption}

\usepackage{booktabs}
\usepackage{nicefrac}
\usepackage{microtype}
\usepackage{xcolor}
\usepackage{colortbl}

\theoremstyle{plain}
\newtheorem{theorem}{Theorem}[section]
\newtheorem{proposition}[theorem]{Proposition}
\newtheorem{lemma}[theorem]{Lemma}
\newtheorem{corollary}[theorem]{Corollary}
\theoremstyle{definition}
\newtheorem{definition}[theorem]{Definition}

\theoremstyle{remark}
\newtheorem{remark}[theorem]{Remark}

\iclrpreprintcopy
\title{Zero Flux: Flow-Based Comparison of High-Dimensional Discrete Distributions}

\author{
Leyang Wang$^{1,2}$ \quad
Yakun Wang$^{3}$ \quad
Song Liu$^{3}$ \quad
Taiji Suzuki$^{1,2}$ \\
\\
$^{1}$Department of Mathematical Informatics, The University of Tokyo, Japan \\
$^{2}$Center for Advanced Intelligence Project, RIKEN, Japan \\
$^{3}$School of Mathematics, University of Bristol, United Kingdom \\
\texttt{\href{mailto:leyang@g.ecc.u-tokyo.ac.jp}
{leyang@g.ecc.u-tokyo.ac.jp}} 
}

\begin{document}

\maketitle
\begin{abstract}
Comparing two high-dimensional discrete distributions has always been a challenging task due to 
the exponentially growing state space and complex changes in interactions. 
A recent work suggests comparing distributions through
a vector field trained using flow matching between two continuous distributions. The resulting vector field at mid-point vanishes if and only if two distributions identical. However, such a flow-based criterion does not naturally apply to discrete distributions. 
We extend this principle to the discrete domain and introduce the
\emph{Zero Flux} criterion, a discrepancy based on local probability fluxes.
Under independent coupling, we show that all local probability fluxes vanish at the midpoint if and only if two distributions are the same. 
This discrepancy decomposes the joint distributional difference into smaller, local contributions and can be efficiently estimated from samples. 
We establish finite sample error bounds for our estimator.
Experiments on synthetic and real categorical data demonstrate
reliable recovery of sparse dependence signals and stable tracking
of distribution shifts in high dimensions.
\end{abstract}

\section{Introduction}
\label{sec:introduction}

Comparing two distributions from samples is a fundamental problem in
statistics and machine learning. For high-dimensional discrete data, the
state space grows exponentially and distributional differences may arise
through changes in complex joint dependence
\citep{pmlr-v258-meel25a}. Existing approaches use kernel distances
\citep{szekely2013energy,gretton2012kernel}, discriminative models
\citep{lopez-paz2017revisiting,kim2021classification}, likelihood-based
models \citep{usman2020log}, or discrete statistical operators
\citep{yang2018goodness,meng2022concrete,amin2023kernelized}.
Each represents distributional differences through a particular kernel,
geometry, or statistical model. Their performance can therefore be
sensitive to how the discrete structure is modelled.

Flow-based models \citep{liuflow,lipman2023flow} learn velocity fields that transform simple reference distributions into complex data distributions. Beyond data generation, Zero Flow methods \citep{wang2026zero,wang2026zero1,shlenskii2026midpoint} uses dynamics learned from endpoint samples to compare distributions, with the midpoint velocity field vanishing under independent coupling if and only if the two distributions are identical. 
Extending this idea to discrete data through a continuous embedding may be suboptimal: although an embedding can assign points in Euclidean space to discrete states, the induced dynamics generally pass through points outside of the embedded discrete set, whose geometry is not associated with the original discrete space. This motivates comparison directly through a discrete analogue of the zero flow criterion.

In this work, we extend the Zero Flow criterion to the discrete domain and
introduce \emph{Zero Flux}. Unlike continuous Flow Matching, which uses a velocity field, Discrete Flow Matching \citep{gat2024discrete,campbell2024generative} describes distributional evolution on a discrete state space through detecting changes in
individual coordinates. We characterize the change using local probability fluxes, which
measure the net probability mass transferred when one coordinate changes. We prove that, at the midpoint \(t=0.5\), all local fluxes vanish
if and only if the endpoint distributions are identical. We then combine
these local fluxes into an overall discrepancy that quantifies the differences
between the two distributions and can be estimated from samples.

A key feature of Zero Flux is its locality. Beyond an overall discrepancy, it decomposes joint changes into local and more interpretable contributions. The signed flux shows how the direction and strength of a local difference
depend on the surrounding coordinates. For example, two text corpora may have similar marginal word frequencies but different word co-occurrence patterns. These local contributions enable us to inspect changes in word associations, not just marginals. This helps understand how language use differs between the corpora, rather than merely detecting a difference. 

\textbf{Contributions.} We introduce \emph{Zero Flux}, a discrete
extension of Zero Flow. We characterize the dynamics of a discrete
bridge through local probability fluxes, which describe how probability
mass moves between states differing in one coordinate. 
Under independent coupling, we show that two joint
distributions are identical if and only if all local probability
fluxes of an independent coupling process vanish at the midpoint.
We derive a neural estimator with finite sample bounds that
separate sampling and posterior estimation errors.
Experiments on synthetic and real categorical data demonstrate
reliable tracking of structured distribution shifts in high
dimensions and decomposition of the overall discrepancy into interpretable
coordinate contributions.

\section{Background}
\label{sec:setup}

Let \(\mathcal V\) be a finite alphabet and \(\mathbf X_0,\mathbf X_1\in\mathcal V^d\) so that $\mathbf X_0 \sim P_0,$ $\mathbf X_1 \sim P_1,$ and $\mathbf X_0 \perp \mathbf X_1.$
\subsection{Zero Flow Criterion}
Zero Flow
\citep{wang2026zero,wang2026zero1,shlenskii2026midpoint}
characterizes distributional equality through time-reversal symmetry
of the dynamics induced by independently coupled endpoints:
\begin{equation}
    \mathbf X_t=(1-t)\mathbf X_0+t\mathbf X_1,
    \qquad
    \mathbf v_t(\mathbf x)
    =
    \mathbb E\!\left[
        \mathbf X_1-\mathbf X_0
        \,\middle|\,
        \mathbf X_t=\mathbf x
    \right].
\end{equation}
Exchanging the endpoints maps \(t\) to \(1-t\) and reverses the
displacement. The resulting symmetry
\begin{equation}
    \mathbf v_t(\mathbf x)=-\mathbf v_{1-t}(\mathbf x)
\end{equation}
along the path provides a criterion for \(P_0=P_1\).
At the midpoint \(t=0.5\), this symmetry reduces to
\(\mathbf v_{0.5}(\mathbf x)=\mathbf 0\), expressing directional
balance conditioned on a symmetric intermediate state.

\citet{wang2026zero} use this principle for representation learning,
and \citet{wang2026zero1} extend it to two-sample testing.
\citet{shlenskii2026midpoint} introduce a coin-flip probability path
based on this principle for training one-step generative models.
In this work, we develop a native discrete extension of the Zero Flow
criterion for comparing high-dimensional discrete distributions.

\subsection{Discrete Flow Matching}
\label{sec:dfm}

Discrete Flow Matching
\citep{gat2024discrete,campbell2024generative}
constructs probability paths directly on discrete state spaces.
Let \((\mathbf X_0,\mathbf X_1)\sim P_0\times P_1\) be independently coupled. Given a fixed pair
\((\mathbf x_0,\mathbf x_1)\), we form an intermediate state by
independently choosing its value from
\(\mathbf x_0\) with probability \(1-t\) or from \(\mathbf x_1\)
with probability \(t\) for each coordinate. This gives the conditional bridge
\begin{equation}
\label{eq:conditional-bridge}
    q_t(\mathbf x\mid\mathbf x_0,\mathbf x_1)
    =
    \prod_{i=1}^d
    \left[
        (1-t)\,\mathbf 1\{x_i=x_0^i\}
        +
        t\,\mathbf 1\{x_i=x_1^i\}
    \right],
    \qquad t\in[0,1].
\end{equation}
Averaging over the endpoint pair gives the marginal distribution
at time \(t\):
\begin{equation}
\label{eq:bridge-marginal}
    p_t(\mathbf x)
    =
    \mathbb E_{(\mathbf X_0,\mathbf X_1)\sim P_0\times P_1}
    \left[
        q_t(\mathbf x\mid\mathbf X_0,\mathbf X_1)
    \right].
\end{equation}
The resulting path satisfies \(p_0=P_0\) and \(p_1=P_1\).
Although the coordinate choices are conditionally independent,
they share the same endpoint pair, allowing \(p_t\) to retain
dependence across coordinates. 
The bridge also has an endpoint exchange symmetry:
\begin{equation}
\label{eq:bridge-reversal}
    q_t(\mathbf x\mid\mathbf x_0,\mathbf x_1)
    =
    q_{1-t}(\mathbf x\mid\mathbf x_1,\mathbf x_0).
\end{equation}
In particular, at \(t=1/2\), exchanging the two endpoints leaves
the conditional bridge unchanged. This midpoint symmetry underlies
the Zero Flux criterion introduced in the next section. A more detailed introduction to DFM is provided in
Appendix~\ref{app:dfm-background}.
\section{Zero Flux Criterion}
\label{sec:Zero Flux}

We extend the Zero Flow criterion \citep{wang2026zero,wang2026zero1,shlenskii2026midpoint} to the
discrete bridge in Section~\ref{sec:dfm} using local probability
fluxes.
These fluxes govern the evolution of the bridge distribution.
Under independent coupling, we show that all fluxes
vanish at \(t=0.5\) if and only if the joint endpoint
distributions are identical.

\subsection{Local Probability Flux and the Continuity Equation}

In continuous space, a velocity field \(\mathbf v_t\) describes
infinitesimal motion, and \(p_t\mathbf v_t\) is the associated
probability current.
Categorical states admit no intrinsic infinitesimal displacement (i.e., the gradient or vector field is not defined).
We therefore use local probability fluxes to describe net
probability transfer between states that differ in one coordinate.

Since \(v\neq x_i\),
\(\{X_1^i=v,\mathbf X_t=\mathbf x\}\) requires coordinate \(i\)
to be copied from \(X_0\), which occurs with probability \(1-t\).
Conversely,
\(\{X_0^i=v,\mathbf X_t=\mathbf x\}\) requires coordinate \(i\)
to be copied from \(X_1\), which occurs with probability \(t\).
We therefore divide the two probabilities by \(1-t\) and \(t\),
respectively. We define the local probability flux as their difference.

\begin{definition}[Local Probability Flux]
\label{def:posterior-flux}
For \(t\in(0,1)\), \(\mathbf x\in\mathcal V^d\) with
\(p_t(\mathbf x)>0\), \(i\in[d]\), and
\(v\in\mathcal V\setminus\{x_i\}\), define the local probability
flux from \(\mathbf x\) to \(\mathbf x^{i\to v}\) by
\begin{equation}
\label{eq:posterior-flux}
    \mathcal F_{t,i}(\mathbf x,v)
    :=
    p_t(\mathbf x)
    \left[
        \frac{
            \mathbb P(X_1^i=v\mid\mathbf X_t=\mathbf x)
        }{1-t}
        -
        \frac{
            \mathbb P(X_0^i=v\mid\mathbf X_t=\mathbf x)
        }{t}
    \right].
\end{equation}
\end{definition}
We set \(\mathcal F_{t,i}(\mathbf x,v)=0\) whenever
\(p_t(\mathbf x)=0\).
The same single coordinate edge can also be viewed in the opposite
direction, from $\mathbf x^{i\to v}$ back to $\mathbf x$.
The corresponding flux then changes sign.
\begin{proposition}[Flux Antisymmetry]
\label{prop:edge_antisymmetry}
For every \(t\in(0,1)\), \(\mathbf x\in\mathcal V^d\),
\(i\in[d]\), and \(v\in\mathcal V\setminus\{x_i\}\),
\begin{equation}
    \mathcal F_{t,i}(\mathbf x,v)
    =
    -\mathcal F_{t,i}(\mathbf x^{i\to v},x_i).
\end{equation}
\end{proposition}

\begin{wrapfigure}[16]{r}{0.5\textwidth}
\vspace{-8pt} 
    \centering
    \includegraphics[width=\linewidth]{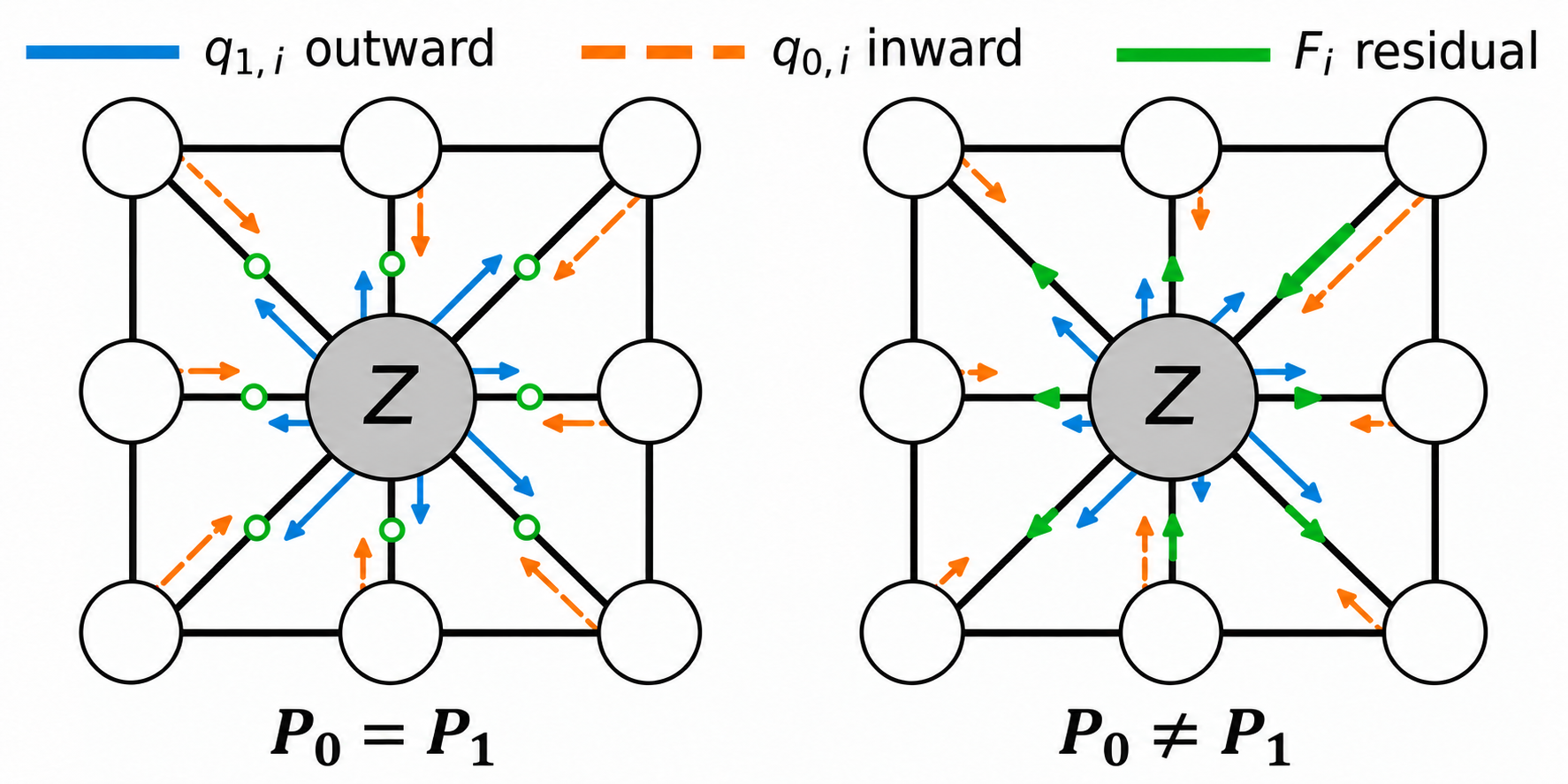}
    \caption{\textbf{Local probability fluxes at the midpoint.}
    Each surrounding state differs from \(\mathbf z\) in one coordinate.
    Blue and orange arrows represent opposing posterior contributions
    from the two endpoints; green arrows show the resulting fluxes.
    All fluxes vanish for identical
    distributions (left), while non-identical distributions produce
    nonzero fluxes at some midpoint states (right).}
    \label{fig:illu}
\end{wrapfigure}

Probability outflow from one state is therefore matched by identical
inflow into the other.
Summing the signed fluxes over all coordinate replacements gives
the net probability outflow from \(\mathbf x\).
For the discrete bridge, differentiating \(p_t\) yields the
following continuity equation.

\begin{proposition}[Continuity Equation]
\label{prop:continuity}
For every \(t\in(0,1)\) and \(\mathbf x\in\mathcal V^d\),
\begin{equation}
\label{eq:continuity}
    \partial_t p_t(\mathbf x)
    =
    -\sum_{i=1}^d
    \sum_{v\neq x_i}
    \mathcal F_{t,i}(\mathbf x,v).
\end{equation}
\end{proposition}

Proofs of both propositions are given in Appendix~\ref{proof:ce}.
Consequently, if all fluxes vanish for every \(t\in(0,1)\),
the bridge distribution is constant and \(P_0=P_1\).

\subsection{Midpoint Zero Flux Criterion}

Even when \(P_0=P_1\), copying each coordinate from either
endpoint can alter dependence among coordinates.
The fluxes may therefore be nonzero away from the midpoint.
At \(t=0.5\), swapping the endpoints leaves the conditional
bridge unchanged by \eqref{eq:bridge-reversal}.
The two posterior contributions in \eqref{eq:posterior-flux}
are exchanged and have the same scaling factor \(2\).
Writing \(\mathbf Z:=\mathbf X_{0.5}\), the flux simplifies to
\begin{equation}
\label{eq:midpoint_joint_flux}
    \mathcal F_{0.5,i}(\mathbf z,v)
    =
    p_{0.5}(\mathbf z) \left[2\,\mathbb P(X_1^i=v|\mathbf Z=\mathbf z)
    -
    2\,\mathbb P(X_0^i=v|\mathbf Z=\mathbf z)\right].
\end{equation}

\textbf{Replacement directions across contexts.}
For \(p_{0.5}(\mathbf z)>0\), a positive
\(\mathcal F_{0.5,i}(\mathbf z,v)\) means that \(X_1^i\)
is more likely than \(X_0^i\) to equal \(v\), given
\(\mathbf Z=\mathbf z\).
The net probability transfer is then from \(\mathbf z\) to
\(\mathbf z^{i\to v}\); a negative value indicates transfer
in the opposite direction.
Zero flux means that the two posterior probabilities agree.
The flux magnitude is their absolute difference multiplied
by \(2p_{0.5}(\mathbf z)\).
Each transition changes only coordinate \(i\), while both
posterior probabilities condition on the complete midpoint state.
For fixed \(z_i\) and \(v\), changing the remaining coordinates
can therefore change the magnitude or reverse the direction
of the flux.

As shown in Figure~\ref{fig:illu}, when \(P_0=P_1\), swapping the independent endpoints leaves
their joint distribution unchanged.
The two probabilities in \eqref{eq:midpoint_joint_flux}
are therefore identical, so all midpoint fluxes vanish.
The following theorem establishes the converse as well.

\begin{theorem}[Zero Flux Criterion]
\label{thm:midpoint}
Under independent coupling, \(P_0=P_1\) if and only if
\[
    \mathcal F_{0.5,i}(\mathbf z,v)=0
    \quad
    \text{for all }
    \mathbf z\in\mathcal V^d,\quad
    i\in[d],\quad
    v\in\mathcal V\setminus\{z_i\}.
\]
\end{theorem}

See Appendix~\ref{proof:mid} for the proof.
When all midpoint fluxes vanish, \(X_0^i\) and \(X_1^i\)
have the same posterior distribution given \(\mathbf Z=\mathbf z\).
This holds for every coordinate \(i\) and every supported
midpoint \(\mathbf z\).
Using the factorization of the conditional bridge and independence
of the endpoints, the proof extends agreement inductively from
univariate marginals to joint marginals of increasing order.
This establishes equality of the full joint distributions.

The criterion therefore detects differences in joint dependence
even when all univariate marginals are the same.
Its context-dependent replacement directions further show how these
differences vary with the remaining coordinates.
The next section turns this criterion into a sample-based discrepancy
while preserving coordinate-wise localization and local directional information.
\section{A Sample-Based Zero Flux Discrepancy}
\label{sec:sample-discrepancy}

We normalize local probability fluxes to express the bridge's
relative probability change.
At the midpoint, their representation through categorical posteriors
supports estimation of an overall discrepancy, coordinate contributions,
and replacement directions that depend on context.

\subsection{Normalized Flux and Population Discrepancy}

For \(t\in(0,1)\) and \(p_t(\mathbf x)>0\), we normalize
the flux by the probability mass at its conditioning state:
\begin{equation}
\label{eq:normalized-flux-general}
    F_{t,i}(\mathbf x,v)
    :=
    \frac{\mathcal F_{t,i}(\mathbf x,v)}
         {p_t(\mathbf x)},
    \qquad v\neq x_i.
\end{equation}
Dividing the continuity equation \eqref{eq:continuity}
by \(p_t(\mathbf x)\) gives
\begin{equation}
\label{eq:log-density-rate}
    \partial_t\log p_t(\mathbf x)
    =
    \frac{\partial_t p_t(\mathbf x)}{p_t(\mathbf x)}
    =
    -\sum_{i=1}^d
    \sum_{v\neq x_i}
    F_{t,i}(\mathbf x,v).
\end{equation}
Thus, \(F_{t,i}(\mathbf x,v)\) measures the signed net probability
outflow associated with the replacement \(x_i\to v\), per unit
probability mass at \(\mathbf x\).
Its negative gives the contribution of this replacement to the
rate of change of \(\log p_t(\mathbf x)\).
Positive normalized flux contributes to a relative decrease
in the state's probability, and negative flux to a relative increase. At the midpoint, let
\[
q_{0,i}^\ast(v\mid\mathbf z)
:=
\mathbb P(X_0^i=v\mid\mathbf Z=\mathbf z),
\qquad q_{1,i}^\ast(v\mid\mathbf z)
:=
\mathbb P(X_1^i=v\mid\mathbf Z=\mathbf z),
\]
denote the exact posteriors of \(X_0^i\) and \(X_1^i\)
given \(\mathbf Z=\mathbf z\).
The normalized midpoint flux becomes
\begin{equation}
\label{eq:true-normalized-flux}
    F_i^\ast(\mathbf z,v)
    :=
    F_{0.5,i}(\mathbf z,v)
    =
    2\left[
        q_{1,i}^\ast(v\mid\mathbf z)
        -
        q_{0,i}^\ast(v\mid\mathbf z)
    \right],
    \qquad v\neq z_i.
\end{equation}
We set \(F_i^\ast(\mathbf z,v)=0\) whenever
\(p_{0.5}(\mathbf z)=0\).
On the midpoint support, normalization preserves the sign and
pointwise zero set of the original flux.
The factor \(p_{0.5}(\mathbf z)\) cancels.
The normalized flux lies in \([-2,2]\) and can be learned
from the observed samples of \(X_0^i\) and \(X_1^i\).

We define the population Zero Flux discrepancy by averaging the
summed squared normalized fluxes under the midpoint distribution:
\begin{equation}
\label{eq:population-discrepancy}
    D^\ast(P_0,P_1)
    :=
    \mathbb E_{\mathbf Z\sim p_{0.5}}
    \left[
        \sum_{i=1}^d
        \sum_{v\neq Z^i}
        F_i^\ast(\mathbf Z,v)^2
    \right].
\end{equation}
In \eqref{eq:log-density-rate}, the negative sum of normalized
fluxes gives the net relative probability change.
Here, squaring before aggregation prevents fluxes of opposite
signs from cancelling.
The midpoint expectation weights each context by its probability.
The resulting discrepancy satisfies the following properties.

\begin{proposition}[Discrepancy Properties]
\label{prop:population-properties}
The following properties hold:
\begin{enumerate}
    \item \textbf{Identifiability:}
    \(D^\ast(P_0,P_1)=0\) if and only if \(P_0=P_1\).

    \item \textbf{Non-negativity and boundedness:}
    \(0\leq D^\ast(P_0,P_1)\leq 4d\).

    \item \textbf{Symmetry:}
    \(D^\ast(P_0,P_1)=D^\ast(P_1,P_0)\).
\end{enumerate}
\end{proposition}

See Appendix~\ref{proof:population-properties} for the proof.

\subsection{Global Comparison and Local Decomposition}
\label{sec:local-readouts}

The Zero Flux discrepancy provides a global measure of the difference
between \(P_0\) and \(P_1\). Since it sums squared fluxes over
coordinates, it can also be decomposed into coordinate-wise
discrepancies:
\begin{equation}
\label{eq:coordinate-discrepancy}
\begin{aligned}
    D^\ast(P_0,P_1)
    &=\sum_{i=1}^d D_i^\ast(P_0,P_1),\qquad
    D_i^\ast(P_0,P_1)
    :=
    \mathbb E_{\mathbf Z\sim p_{0.5}}
    \left[
        \sum_{v\neq Z^i}
        F_i^\ast(\mathbf Z,v)^2
    \right].
\end{aligned}
\end{equation}
\(D_i^\ast\) measures how much the fluxes associated with
coordinate \(i\) contribute to the overall discrepancy.
The collection \(\{D_i^\ast\}_{i=1}^d\) therefore shows where the
measured distributional difference is concentrated.

The fluxes $F_i$ further show how differences appear within each coordinate.
For a fixed midpoint state \(\mathbf z\),
\(F_i^\ast(\mathbf z,v)\) describes the local difference associated
with replacing \(z_i\) by \(v\).
Its sign indicates the direction of this local difference, while its magnitude
signifies the level of change. 
Because the flux conditions on the complete midpoint state,
the same replacement can behave differently as the remaining
coordinates change.

This decomposition property \eqref{eq:coordinate-discrepancy} is particularly useful when the two distributions
differ only through a subset of coordinates.
Let \(S\subseteq[d]\) denote these coordinates.
We next show that, 
when the remaining coordinates are identically distributed
under \(P_0\) and \(P_1\) and are independent of \(S\),
they do not affect either the local fluxes on \(S\) or the global
discrepancy.

\begin{proposition}[Invariance to Independent Irrelevant Coordinates]
\label{prop:irrelevant-invariance}
Let \(\varnothing\neq S\subsetneq[d]\) and suppose
\[
    P_0=P_{0,S} P_{0,S^c},
    \qquad
    P_1=P_{1,S}P_{1,S^c} \quad\text{where}
    \qquad
    P_{0,S^c}=P_{1,S^c}.
\]
Let \(F_{S,i}^\ast\) denote the normalized midpoint flux
of the discrete bridge between \(P_{0,S}\) and \(P_{1,S}\).
For every supported midpoint
\(\mathbf z\), \(i\in[d]\), and
\(v\in\mathcal V\setminus\{z_i\}\),
\begin{equation}
\label{eq:irrelevant-flux-invariance}
    F_i^\ast(\mathbf z,v)
    =
    \begin{cases}
        F_{S,i}^\ast(\mathbf z_S,v), & i\in S,\\
        0, & i\in S^c.
    \end{cases}
    \Rightarrow
    D_i^\ast(P_0,P_1)
    =0,
     i\in S^c .
\end{equation}
Consequently, $D^\ast(P_0,P_1)
    =D^\ast(P_{0,S},P_{1,S})$.
\end{proposition}
See Appendix~\ref{app:irrelevant-invariance} for the proof. When the coordinates in \(S^c\) are the same under both
distributions and independent of \(S\), they do not contribute to discrepancy,
while the fluxes on \(S\) remain unchanged.
The global score is determined entirely by the coordinates in \(S\),
while their local flux structure is preserved.
In Section~\ref{sec:sparse-field-recovery}, we evaluate whether Zero Flux preserves
the structure as more irrelevant coordinates are added.

\subsection{Empirical Estimation}
\label{sec:posterior-estimation}

Let
\(\mathbf X_{0,j}\stackrel{\mathrm{iid}}{\sim}P_0\) and
\(\mathbf X_{1,k}\stackrel{\mathrm{iid}}{\sim}P_1\),
for \(j,k\in[n]\), with the two endpoint datasets independent.
We independently sample index pairs with replacement
and draw their bridge midpoints:
\[
    (A_\ell,B_\ell)
    \sim
    \operatorname{Unif}([n])\times\operatorname{Unif}([n]),
    \qquad
    \mathbf Z_\ell
    \sim
    q_{0.5}\!\left(
        \,\cdot\mid
        \mathbf X_{0,A_\ell},
        \mathbf X_{1,B_\ell}
    \right).
\]
We train two neural networks
\(\widehat q_{0,i}(\cdot\mid\mathbf z)\) and
\(\widehat q_{1,i}(\cdot\mid\mathbf z)\)
to predict \(X_0^i\) and \(X_1^i\), respectively, from the
complete midpoint state.
The corresponding cross entropy objective is
\begin{equation}
\label{eq:posterior-training}
    \mathcal L
    :=
    -\mathbb E
    \left[
        \sum_{i=1}^d
        \left(
            \log \widehat q_{0,i}(X_0^i\mid\mathbf Z)
            +
            \log \widehat q_{1,i}(X_1^i\mid\mathbf Z)
        \right)
    \right],
\end{equation}
where the expectation is over independent endpoints
\((\mathbf X_0,\mathbf X_1)\sim P_0\times P_1\) and their
bridge midpoint.
Categorical cross entropy is strictly proper.
Its population minimizers over all categorical predictors are
the exact posteriors \(q_{0,i}^\ast\) and \(q_{1,i}^\ast\)
on the midpoint support.
These are the posteriors underlying the criterion in
Equation~\ref{eq:true-normalized-flux}.

We then estimate the coordinate contributions
and overall discrepancy from \(m\) evaluation midpoints:
\begin{equation}
\label{eq:empirical-discrepancy}
\begin{aligned}
    \widehat D_i
    &:=
    \frac{4}{m}
    \sum_{\ell=1}^m
    \sum_{v\neq Z_\ell^i}
    \left[
        \widehat q_{1,i}(v\mid\mathbf Z_\ell)
        -
        \widehat q_{0,i}(v\mid\mathbf Z_\ell)
    \right]^2,\qquad
    \widehat D_{n,m}
    :=
    \sum_{i=1}^d \widehat D_i.
\end{aligned}
\end{equation}
The same fitted heads provide the normalized flux field,
coordinate contributions, and overall discrepancy, using
\(d(|\mathcal V|-1)\) normalized flux components per sampled midpoint.

\section{Theoretical Analysis}
\label{sec:theoretical-analysis}

We now study how accurately \(\widehat D_{n,m}\) estimates the
population discrepancy \(D^\ast\).
The error arises from endpoint and midpoint sampling fluctuations
and from posterior estimation.
We quantify the latter by the average excess cross entropy:
\begin{align}
    \widehat{\mathcal E}_0
    &:=
    \frac{1}{md}
    \sum_{\ell=1}^m
    \sum_{i=1}^d
    \operatorname{KL}\!\left(
        q_{0,i}^\ast(\cdot\mid\mathbf Z_\ell)
        \,\middle\|\,
        \widehat q_{0,i}(\cdot\mid\mathbf Z_\ell)
    \right); \notag\\
    \widehat{\mathcal E}_1
    &:=
    \frac{1}{md}
    \sum_{\ell=1}^m
    \sum_{i=1}^d
    \operatorname{KL}\!\left(
        q_{1,i}^\ast(\cdot\mid\mathbf Z_\ell)
        \,\middle\|\,
        \widehat q_{1,i}(\cdot\mid\mathbf Z_\ell)
    \right).
    \label{eq:empirical-excess-cross-entropy}
\end{align}
Each KL term measures the additional cross entropy of the
learned posterior relative to the exact posterior at a sampled midpoint. To connect posterior learning with discrepancy estimation,
we bound how these excess losses affect \(\widehat D_{n,m}\).
The bound also accounts for the dependence among the \(n^2\)
endpoint pairs, which share the same \(2n\) observations.

\begin{theorem}[Error Decomposition]
\label{thm:finite-sample-discrepancy}
For every \(\delta\in(0,1)\), with probability at least \(1-\delta\)
over the endpoint datasets, sampled index pairs, and masks,
the following bound holds simultaneously for all categorical
posterior networks with finite \(\widehat{\mathcal E}_0\) and
\(\widehat{\mathcal E}_1\):
\begin{align}
    \left|
        D^\ast-\widehat D_{n,m}
    \right|
    \leq{}&
    4d
    \sqrt{
        \left(
            \frac{1}{n}+\frac{1}{2m}
        \right)
        \log\frac{2}{\delta}
    }+
    4\sqrt{d\widehat D_{n,m}}
    \left(
        \sqrt{\widehat{\mathcal E}_0}
        +
        \sqrt{\widehat{\mathcal E}_1}
    \right)+
    4d
    \left(
        \sqrt{\widehat{\mathcal E}_0}
        +
        \sqrt{\widehat{\mathcal E}_1}
    \right)^2.
    \label{eq:finite-sample-discrepancy-bound}
\end{align}
\end{theorem}

See Appendix~\ref{app:theoretical-analysis} for the proof.
The first term in \eqref{eq:finite-sample-discrepancy-bound}
controls fluctuations from endpoint sampling and midpoint
Monte Carlo evaluation.
The remaining two terms bound the effect of posterior estimation
error on the discrepancy.
For categorical posterior networks, \(\widehat D_{n,m}\leq 8d\).
Thus, for fixed \(d\), the bound implies
\(\widehat D_{n,m}\xrightarrow{p}D^\ast\) whenever
\(n,m\to\infty\) and
\(\widehat{\mathcal E}_0,\widehat{\mathcal E}_1\xrightarrow{p}0\).

We next ask how large the empirical discrepancy can be when
\(P_0=P_1\).
This characterizes the apparent separation that can arise
when comparing samples from the same distribution.

\begin{corollary}[Exact null]
\label{cor:null-discrepancy}
If \(P_0=P_1\), then, almost surely under the sampling construction,
all categorical posterior networks with finite
\(\widehat{\mathcal E}_0\) and \(\widehat{\mathcal E}_1\) satisfy
\begin{equation}
    \widehat D_{n,m}
    \leq
    4d
    \left(
        \sqrt{\widehat{\mathcal E}_0}
        +
        \sqrt{\widehat{\mathcal E}_1}
    \right)^2.
    \label{eq:null-discrepancy-cross-entropy}
\end{equation}
\end{corollary}

See Appendix~\ref{app:theoretical-analysis} for the proof.
Under the null, all exact normalized fluxes vanish at every
midpoint with positive probability.
The bound therefore depends only on posterior estimation error.

\begin{remark}
Posterior networks with a softmax parametrization and finite logits
assign positive probability to every category, ensuring finite
\(\widehat{\mathcal E}_0\) and \(\widehat{\mathcal E}_1\).
Theorem~\ref{thm:finite-sample-discrepancy} combines concentration
of the empirical discrepancy computed with exact posteriors
with a deterministic bound in terms of posterior estimation errors.
The deterministic bound holds for arbitrary posterior networks,
including those trained and evaluated on the same samples.
\end{remark}

\section{Related Work}
\paragraph{Comparing Discrete Distributions.}
Sample-based distribution comparisons differ in the statistical object used
to represent distributional separation. Kernel and feature-based discrepancies,
such as MMD \citep{gretton2012kernel} and learned-kernel variants
\citep{liu2020learning}, compare distributions through differences in feature
mean embeddings. Classifier and density ratio based approaches quantify how
well two sample populations can be distinguished or estimate their relative
density directly
\citep{lopez-paz2017revisiting,kim2021classification,kanamori2012fdivergence,
gerber2023minimax}. Model-based approaches instead rely on probabilistic fits:
likelihood-ratio discrepancies use explicit likelihoods
\citep{usman2020log}, while latent-variable models such as VAEs \citep{kingma2013auto} and IWAEs
\citep{burda2015importance} provide tractable variational
surrogates.

Discrete Stein operators \citep{yang2018goodness,amin2023kernelized}
and score matching \citep{meng2022concrete} exploit local neighborhood
structure. Zero Flux connects this perspective with comparison through
flow dynamics, linking local posterior balance under the midpoint
distribution to equality of the complete endpoint distributions.
 
\section{Experiments}
We test whether Zero Flux can recover sparse dependence changes
between high dimensional discrete distributions and how this ability supports
distributional comparison across different tasks.
Implementation details and additional results are provided in
Appendices~\ref{app:experimental-implementation}
and~\ref{app:additional}.
\begin{figure}[t]
    \centering
    \begin{subfigure}[t]{0.26\linewidth}
        \vspace{0pt}
        \centering
        \footnotesize
        \setlength{\tabcolsep}{3pt}
        \renewcommand{\arraystretch}{1.05}
        \resizebox{\linewidth}{!}{\begin{tabular}{llcc}
    \toprule
    Method & $d$ & Sign $\uparrow$ & Cosine $\uparrow$ \\
    \midrule
    \multirow{3}{*}{\shortstack{Zero\\Flux}}
    & 8   & $1.000{\pm}0.000$ & $0.995{\pm}0.001$ \\
    & 64  & $1.000{\pm}0.000$ & $0.983{\pm}0.005$ \\
    & 256 & $1.000{\pm}0.000$ & $0.819{\pm}0.170$ \\
    \midrule
    \multirow{3}{*}{\shortstack{Coord\\MSE}}
    & 8   & $1.000{\pm}0.000$ & $0.942{\pm}0.077$ \\
    & 64  & $1.000{\pm}0.000$ & $0.677{\pm}0.214$ \\
    & 256 & $0.594{\pm}0.038$ & $0.151{\pm}0.104$ \\
    \midrule
    \multirow{3}{*}{\shortstack{Coord\\CE}}
    & 8   & $1.000{\pm}0.000$ & $0.996{\pm}0.001$ \\
    & 64  & $0.603{\pm}0.454$ & $0.396{\pm}0.433$ \\
    & 256 & $0.563{\pm}0.059$ & $0.029{\pm}0.024$ \\
    \bottomrule
\end{tabular}}
        \caption{Local field recovery with three relevant coordinates
and \(d-3\) independent irrelevant coordinates.}
        \label{tab:edge-reversal-ga}
    \end{subfigure}
    \hfill
    \begin{subfigure}[t]{0.72\linewidth}
        \vspace{0pt}
        \centering
        \includegraphics[width=\linewidth]{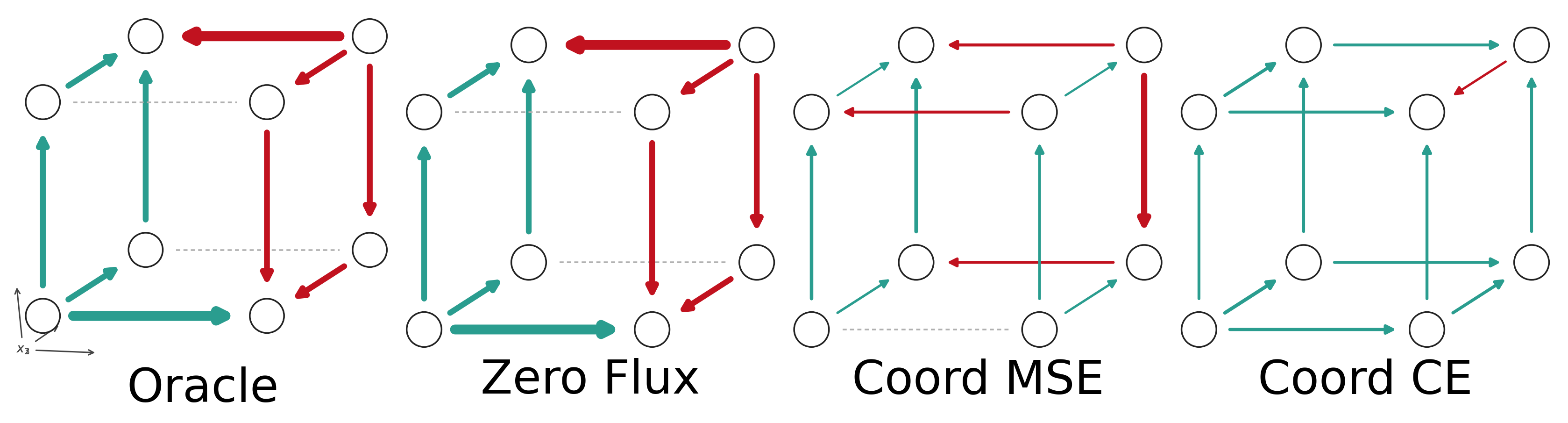}
        \caption{Visualization of Recovered Local fields at \(d=256\).
Vertices represent binary states of the three relevant coordinates;
edges represent single coordinate flips.
Teal and red arrows indicate \(0\to1\) and \(1\to0\), respectively.
Widths encode magnitudes on a common scale.}
        \label{fig:edge_reversal}
    \end{subfigure}
    \caption{Recovery of sparse dependence changes.
Both results show that Zero Flux closely recovers the analytical oracle field as \(d\) increases while Zero Flow lose accuracy in direction and magnitude.}
\end{figure}
\subsection{Sparse Field Recovery in High Dimensions}
\label{sec:sparse-field-recovery}
We test whether Zero Flux can recover sparse dependence changes
from samples as irrelevant coordinates are added as shown in Proposition~\ref{prop:irrelevant-invariance}.
We assess recovery through the local field's directions and magnitudes.

\textbf{Sparse dependence changes.}
We test whether sparse dependence changes remain recoverable
as irrelevant coordinates are added.
For \(\mathbf x\in\{0,1\}^d\), \(d\geq3\), and \(0<\eta<1\), let
\begin{equation}
\label{eq:sparse-recovery-pair}
    P_0(\mathbf x)
    =2^{-d}\left[1-\eta(2x_1-1)(1-x_2-x_3)\right],
    \qquad
    P_1(\mathbf x)
    =2^{-d}\left[1+\eta(2x_1-1)(1-x_2-x_3)\right].
\end{equation}
The distributions share all univariate marginals and differ
only in dependence within \(S=\{1,2,3\}\).
Increasing \(d\) adds independent uniform coordinates with
identical laws under both distributions.
By Proposition~\ref{prop:irrelevant-invariance}, the exact
normalized fluxes remain unchanged on \(S\) and vanish elsewhere.

\textbf{Baselines.}
We compare Zero Flux with two coordinate-wise one-hot Zero Flow
baselines. Define
\[
    \phi_c(\mathbf x)
    :=(\mathbf e_{x_1},\ldots,\mathbf e_{x_d})
    \in\mathbb R^{2d},
    \qquad
    \mathbf M_c
    :=\frac{\phi_c(\mathbf X_0)+\phi_c(\mathbf X_1)}{2},
\]
where \(\mathbf e_0,\mathbf e_1\) are the standard basis vectors
of \(\mathbb R^2\).
Given \(\mathbf M_c\), \emph{Coord MSE} predicts
\(\phi_c(\mathbf X_1)-\phi_c(\mathbf X_0)\) using squared error.
\emph{Coord CE} predicts \(X_0^i\) and \(X_1^i\) using the same
cross entropy as Zero Flux.

\textbf{Metrics.}
We compare flux and velocity estimates using the discrete
log density-ratio gradient for single coordinate flips:
\(G_i(\mathbf z)
    :=
    \log\frac{P_1(\mathbf z^{i\to1-z_i})}
              {P_0(\mathbf z^{i\to1-z_i})}
    -
    \log\frac{P_1(\mathbf z)}{P_0(\mathbf z)}\).
For this family,
\(G_i(\mathbf z)=c_\eta F_i^\ast(\mathbf z,1-z_i)\),
where \(c_\eta=(2/\eta)\log((1+\eta)/(1-\eta))>0\).
This rescaling preserves flux directions and relative magnitudes.
We convert each method's predictions to \(\widehat G_i\)
using the analytic maps in Appendix~\ref{app:ablation-study},
which also provides implementation details. We report two metrics:
\emph{Sign agreement} is the fraction of evaluated components
with \(G_i(\mathbf z)\neq0\) satisfying
\(\widehat G_i(\mathbf z)G_i(\mathbf z)>0\).
\emph{Cosine similarity} measures field alignment over all
evaluated states and coordinates.

Table~\ref{tab:edge-reversal-ga} reports recovery metrics across
dimensions.
Figure~\ref{fig:edge_reversal} shows the fields on \(S\),
averaged over irrelevant coordinates.
All methods recover the target field in low dimensions.
As more irrelevant coordinates are added, Zero Flux maintains
accurate directions and near-target average magnitudes.
In contrast, both coordinate-wise Zero Flow methods lose accuracy as dimension increases.
This highlights the benefit of Zero Flux's native discrete
construction for recovering sparse dependence changes.
The visualization details and additional results are in Appendix~\ref{app:ablation-study}.

\subsection{Tracking Structured Distribution Shifts}

\begin{table*}[t]
\centering
\caption{{\bf Quantitative Comparisons on Distribution Contamination.}
LRD uses a fitted first order Markov likelihood on both Markov
and HMM data, and an autoregressive likelihood on binarized MNIST.
The best results are highlighted in bold, and the second-best results are underlined.}
\label{tab:quantitative_metrics}
\resizebox{0.999\textwidth}{!}{
\begin{tabular}{l*{8}{c}}
\toprule
& \multicolumn{2}{c}{Markov 8d}
& \multicolumn{2}{c}{Markov 128d}
& \multicolumn{2}{c}{HMM 128d}
& \multicolumn{2}{c}{Binarized MNIST 784d} \\
\cmidrule(lr){2-3}
\cmidrule(lr){4-5}
\cmidrule(lr){6-7}
\cmidrule(lr){8-9}
\textbf{Method}
 & AUROC $\uparrow$ & FPR@95\% TPR $\downarrow$
& AUROC $\uparrow$ & FPR@95\% TPR $\downarrow$
& AUROC $\uparrow$ & FPR@95\% TPR $\downarrow$
& AUROC $\uparrow$ & FPR@95\% TPR $\downarrow$ \\
\midrule

MMD
& 0.794 & 0.455
& 0.690 & 0.695
& 0.532 & 0.855
& \underline{0.891} & \underline{0.255} \\

\cmidrule(lr){1-1}
\cmidrule(lr){2-3}
\cmidrule(lr){4-5}
\cmidrule(lr){6-7}
\cmidrule(lr){8-9}

Deep MMD
& 0.797 & 0.440
& 0.661 & 0.665
& \underline{0.632} & \underline{0.690}
& 0.631 & 0.810 \\

\cmidrule(lr){1-1}
\cmidrule(lr){2-3}
\cmidrule(lr){4-5}
\cmidrule(lr){6-7}
\cmidrule(lr){8-9}

Classifier-JS
& \underline{0.977} & \textbf{0.075}
& 0.424 & 0.945
& 0.516 & 0.940
& 0.661 & 0.980 \\

\cmidrule(lr){1-1}
\cmidrule(lr){2-3}
\cmidrule(lr){4-5}
\cmidrule(lr){6-7}
\cmidrule(lr){8-9}

VAE-ELBO
& 0.695 & 0.575
& 0.363 & 0.990
& 0.415 & 0.985
& 0.674 & 0.800 \\

\cmidrule(lr){1-1}
\cmidrule(lr){2-3}
\cmidrule(lr){4-5}
\cmidrule(lr){6-7}
\cmidrule(lr){8-9}

VAE-IWAE
& 0.670 & 0.590
& 0.387 & 0.985
& 0.436 & 0.980
& 0.703 & 0.715 \\

\cmidrule(lr){1-1}
\cmidrule(lr){2-3}
\cmidrule(lr){4-5}
\cmidrule(lr){6-7}
\cmidrule(lr){8-9}

LRD$^\ast$
& \textbf{0.982} & \underline{0.120}
& \textbf{1.000} & \textbf{0.000}
& 0.522 & 0.915
& 0.692 & 0.860 \\

\cmidrule(lr){1-1}
\cmidrule(lr){2-3}
\cmidrule(lr){4-5}
\cmidrule(lr){6-7}
\cmidrule(lr){8-9}

Zero Flow
& 0.958 & 0.185
& 0.861 & 0.410
& 0.563 & 0.970
& 0.881 & 0.355 \\

\midrule
\rowcolor[gray]{0.9}
\textbf{Zero Flux}
& 0.874 & 0.305
& \underline{0.920} & \underline{0.200}
& \textbf{0.702} & \textbf{0.660}
& \textbf{0.919} & \textbf{0.200} \\

\bottomrule
\end{tabular}
}
\end{table*}

\begin{figure}
    \centering
    \includegraphics[width=0.99\linewidth,
    height=0.4\textwidth]{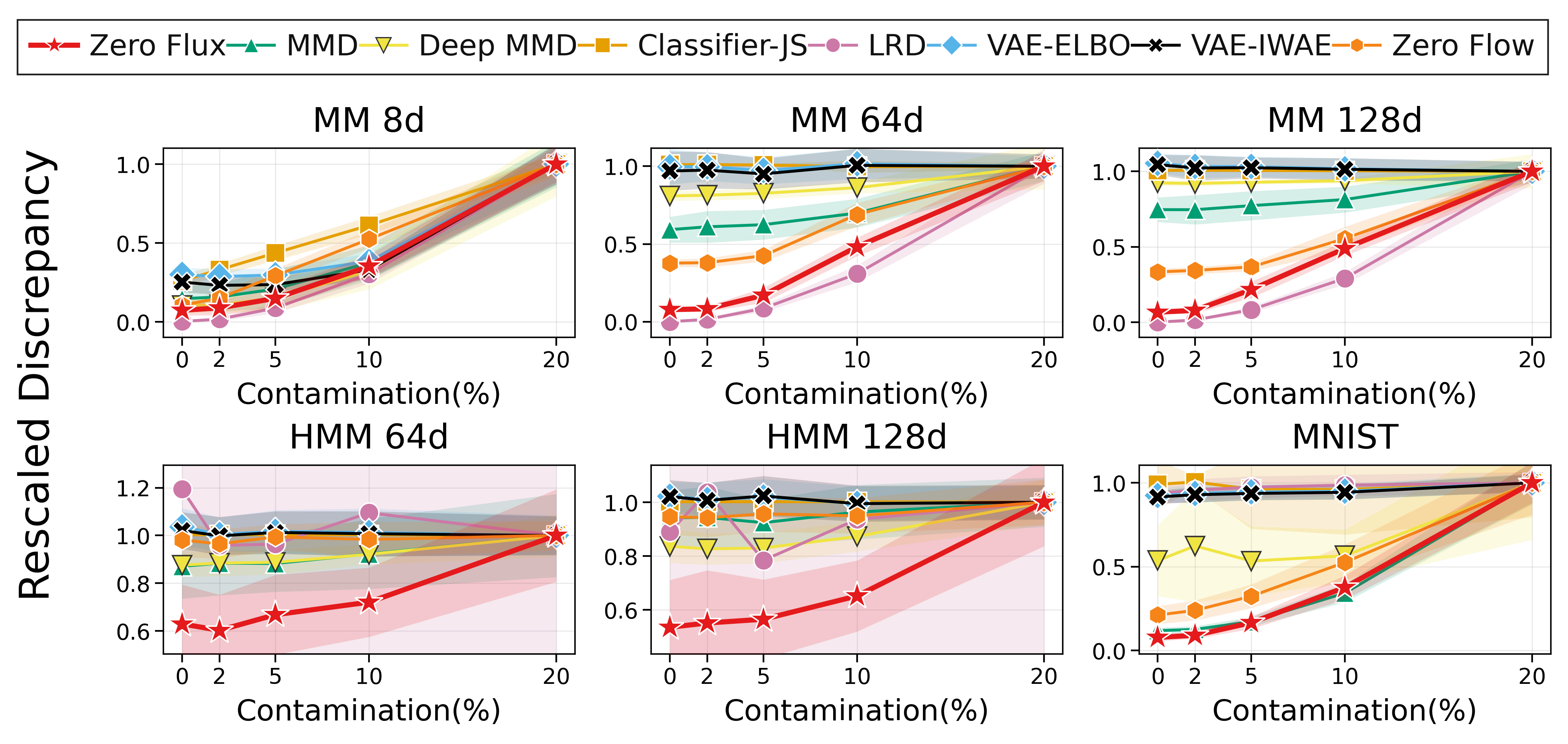}
    \caption{{\bf Contamination-response curves.}
    Results are averaged over 50 independent runs; shaded bands
    show one standard deviation. Each curve is rescaled by its mean
    value at $20\%$ contamination.}
    \label{fig:contamination-response_}
\end{figure}
Having examined local field dependence recovery, we now test whether the
discrepancy 
\eqref{eq:population-discrepancy} yields
reliable scalar measurements of high-dimensional distribution shifts when data is gradually contaminated by samples from a different distribution. More details are in Appendix~\ref{app:contamination-implementation}.

\textbf{Setup.}
For a reference distribution \(P\) and a contaminating distribution
\(R\) on \(\mathcal V^d\), we construct
\[
    Q_{\varepsilon}
    =
    (1-\varepsilon)P+\varepsilon R,
    \qquad
    \varepsilon\in\{0,0.02,0.05,0.1,0.2\}.
\]
Each observation from \(Q_\varepsilon\) is a complete vector
drawn from \(R\) with probability \(\varepsilon\) and from \(P\)
otherwise.
We evaluate sensitivity to dependence changes using stationary
binary Markov chains with \(X^1\sim\mathrm{Bernoulli}(0.5)\) and
\(\mathbb P(X^i=X^{i-1})=s\).
The reference and contaminating distributions differ in persistence
while sharing all univariate marginals.
We also consider an HMM pair with identical adjacent coordinate
correlations but different longer range correlations.
For binarized MNIST \citep{lecun1998gradient}, \(P\) and \(R\) are
the distributions of digits \(6\) and \(9\) in
\(\{0,1\}^{784}\). All methods receive the same two independently sampled datasets, and compute discrepancies from these
samples alone.

\textbf{Baselines and Evaluation Metrics.}
We compare Zero Flux with MMD, Deep MMD \citep{liu2020learning}, Classifier-JS,
VAE-ELBO, VAE-IWAE, Zero Flow and likelihood based approach LRD \citep{usman2020log}. LRD uses the same fitted first order Markov likelihood in both settings.
We evaluate contamination detection through false positives (Type I errors)
and false negatives (Type II errors).
Since both error rates depend on the detection threshold, we report
AUROC to summarize their trade-off across thresholds.
We also report FPR@95\% TPR, which measures the minimum empirical
Type I error rate while keeping the empirical Type II error rate
at or below 5\%.
Both metrics compare each contamination level with the
null case and are averaged equally across levels.

Table~\ref{tab:quantitative_metrics} summarizes contamination
detection performance, while Figure~\ref{fig:contamination-response_}
shows how the discrepancies change with contamination. We also visualise metrics versus $\epsilon$ in Figure~\ref{fig:2exp_metrics}.
Zero Flux and LRD improve on longer Markov sequences, whereas
the other methods deteriorate.
Zero Flux achieves the best detection performance on the HMM
and binarized MNIST.
In particular, it detects longer-range dependence changes in
the HMM, where Markov-based LRD performs near chance. The response curves show that the Zero Flux discrepancy generally
increases with contamination.
Some baselines instead show nearly flat or non-monotonic responses.
Together, these results show that Zero Flux can detect structured
distribution shifts and track their severity.
Its advantage over Zero Flow extends the benefit of
the native discrete construction from local field recovery
to global distribution comparison.
\subsection{Tracking Dependence Degradation in Categorical Tabular Data}
\label{sec:uci}
\begin{wraptable}{r}{0.5\textwidth}
\vspace{-\intextsep} 
\centering
\caption{UCI categorical data results.}
\label{tab:uci-core-metrics}

\renewcommand{\arraystretch}{1.05}
\resizebox{\linewidth}{!}{\begin{tabular}{lccc}
\toprule
Method
& C-index \(\uparrow\)
& Pearson \(r(s,D)\) \(\uparrow\)
& NVI \(\downarrow\) \\
\midrule
MMD
& \(0.930 \pm 0.045\)
& \(0.952 \pm 0.035\)
& \(0.096 \pm 0.077\) \\
MMD-FUSE
& \(0.982 \pm 0.028\)
& \(0.940 \pm 0.065\)
& \(0.002 \pm 0.002\) \\
Cram\'er's \(V\)
& \(0.999 \pm 0.004\)
& \(0.952 \pm 0.010\)
& \(0.001 \pm 0.003\) \\
Zero Flux
& \(0.994 \pm 0.016\)
& \(0.953 \pm 0.040\)
& \(0.002 \pm 0.008\) \\
\bottomrule
\end{tabular}}
\end{wraptable}

We then study
another form of distribution shift in categorical tabular data by
breaking dependence between features.
Using UCI datasets \citep{kelly2019}, we progressively split
features into blocks and independently rematch the blocks across rows.
This disrupts cross-block dependence while preserving every
univariate empirical marginal exactly.
The severity \(s\) is the fraction of feature pairs assigned
to different blocks.

We compare Zero Flux with a pairwise-association discrepancy
based on Cram\'er's \(V\) which aligned with this degradation,
and with MMD and MMD-FUSE \citep{biggs2023mmd}.
We report the following metrics: C-index measures severity ordering, Pearson correlation measures
alignment with sensitivity, and the normalized violation index (NVI)
measures score decreases along the path, normalized by the absolute
difference between the initial and final scores. See Appendix~\ref{app:uci-categorical-dependence} for definition details.

Table~\ref{tab:uci-core-metrics} shows that Zero Flux closely
tracks dependence degradation, approaching the task aligned
Cram\'er's \(V\) reference in ordering and exhibiting few
downward reversals.
MMD-FUSE also tracks the path reliably, while fixed kernel MMD
shows more ordering violations. 
\subsection{Tracing Partisan Distributional Separation over Time in U.S. Nationscape}

We use the overall discrepancy and its interpretable coordinate decomposition
from Section~\ref{sec:local-readouts} to track partisan separation
over time and identify policy-item contributions.
We pool the 77 regular waves of Nationscape
\citep{tausanovitch2021nationscape} into 19 monthly bins.
For each month \(t\), \(P_t\) and \(Q_t\) are the survey weighted
empirical joint distributions of a fixed 14 item policy response
vector among self-identified Democrats and Republicans.
A single month-conditioned Zero Flux model estimates the
discrepancy per coordinate \(D_{\mathrm{ZF}}(P_t,Q_t)/d\),
with \(d=14\).
Survey weighting follows \citet{holliday2021democracy};
preprocessing and implementation details are given in
Appendix~\ref{app:nationscape}.

\begin{figure}
    \centering
    \includegraphics[
        width=0.99\linewidth,
        height=0.22\textwidth
    ]{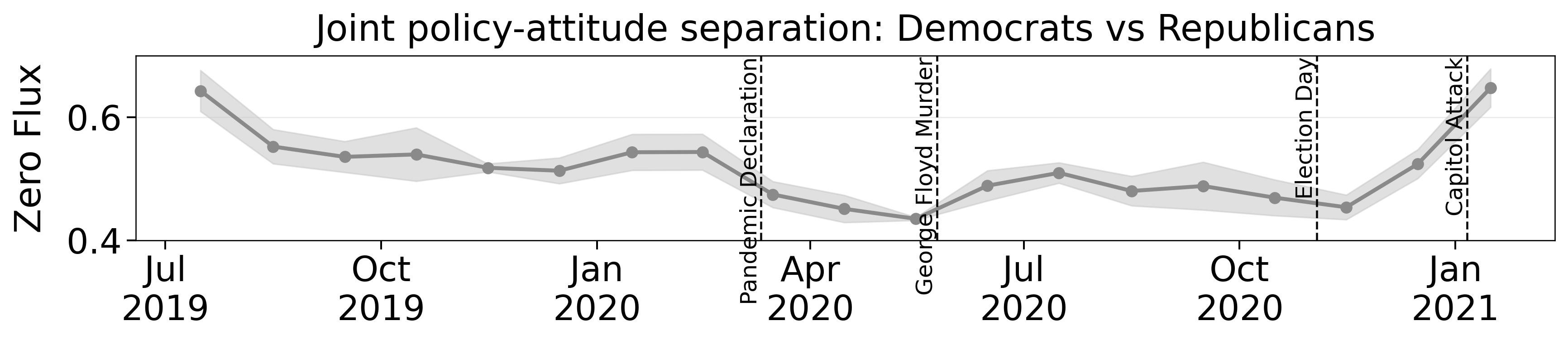}
    \caption{Monthly Zero Flux discrepancy between
    the weighted joint policy response distributions of
    self-identified Democrats and Republicans.
    Dashed lines mark selected major events.
    Shaded bands show standard deviation of independent runs.}
    \label{fig:nationscape}
\end{figure}

Figure~\ref{fig:nationscape} shows lower separation around the
onset of the COVID-19 pandemic and an increase following
the murder of George Floyd.
The policy-item contributions
(Figure~\ref{fig:Nationscape_Heatmap}) complement this trajectory
by locating where the measured separation is concentrated.
The border wall item contributes most in every month,
consistent with previously documented partisan divisions over
this policy \citep{kim2020comparativer,sides2023bitter}. Additional comparison is provided in Appendix~\ref{app:nation_comp}.
\section{Conclusions}
We introduced Zero Flux, a native method that characterizes the equality of discrete 
distributions through local probability fluxes under
independent coupling. 
We also show that this criterion can be further developed into a statistical discrepancy, which can be estimated through 
categorical posterior estimation. This procedure yields an effective sample-based discrepancy. 
We establish finite-sample guarantees and demonstrate recovery of
sparse dependence changes and reliable comparison of high-dimensional
categorical distributions.
\section*{Acknowledgement}
We thank Yuan Lu for helpful feedback. LW was partially supported by JSPS KAKENHI (24K02905).
TS was partially supported by JST CREST (PMJCR2015) and JSPS KAKENHI (25H01107). This work was supported by JST ERATO Grant Number JPMJER2601, Japan. This research is supported by the National Research Foundation, Singapore and the Ministry of Digital Development and Information under the AI Visiting Professorship Programme (award number AIVP-2024-004). Any opinions, findings and conclusions or recommendations expressed in this material are those of the author(s) and do not reflect the views of National Research Foundation, Singapore and the Ministry of Digital Development and Information.

\bibliographystyle{apalike}
\bibliography{references}
\clearpage
\appendix
\section{Background on Discrete Flow Matching}
\label{app:dfm-background}

Discrete Flow Matching (DFM)
\citep{gat2024discrete,campbell2024generative} constructs a probability
path on a discrete state space and a continuous-time Markov chain (CTMC)
whose marginals follow that path. We describe the coordinatewise linear
bridge used here. Let $\mathcal V$ be a finite alphabet,
$\mathbf X_0,\mathbf X_1\in\mathcal V^d$, and
$(\mathbf X_0,\mathbf X_1)\sim P_0\times P_1$.

For a fixed pair $(\mathbf x_0,\mathbf x_1)$, each coordinate of
$\mathbf X_t$ takes its value from $\mathbf x_0$ with probability $1-t$
and from $\mathbf x_1$ with probability $t$, independently across
coordinates. Its conditional distribution and the resulting marginal
path are
\begin{equation}
\label{eq:app-dfm-marginal-path}
\begin{aligned}
q_t(\mathbf x\mid\mathbf x_0,\mathbf x_1)
&=\prod_{i=1}^d
\left[(1-t)\mathbf 1\{x_i=x_0^i\}
       +t\mathbf 1\{x_i=x_1^i\}\right],\\
p_t(\mathbf x)
&=\mathbb E_{(\mathbf X_0,\mathbf X_1)\sim P_0\times P_1}
\bigl[q_t(\mathbf x\mid\mathbf X_0,\mathbf X_1)\bigr].
\end{aligned}
\end{equation}
Thus $p_0=P_0$ and $p_1=P_1$. The coordinate choices are independent
conditional on the complete endpoint pair; averaging over that pair
allows $p_t$ to carry dependence among coordinates.

\begin{figure}[!b]
    \centering
    \includegraphics[width=0.72\linewidth]{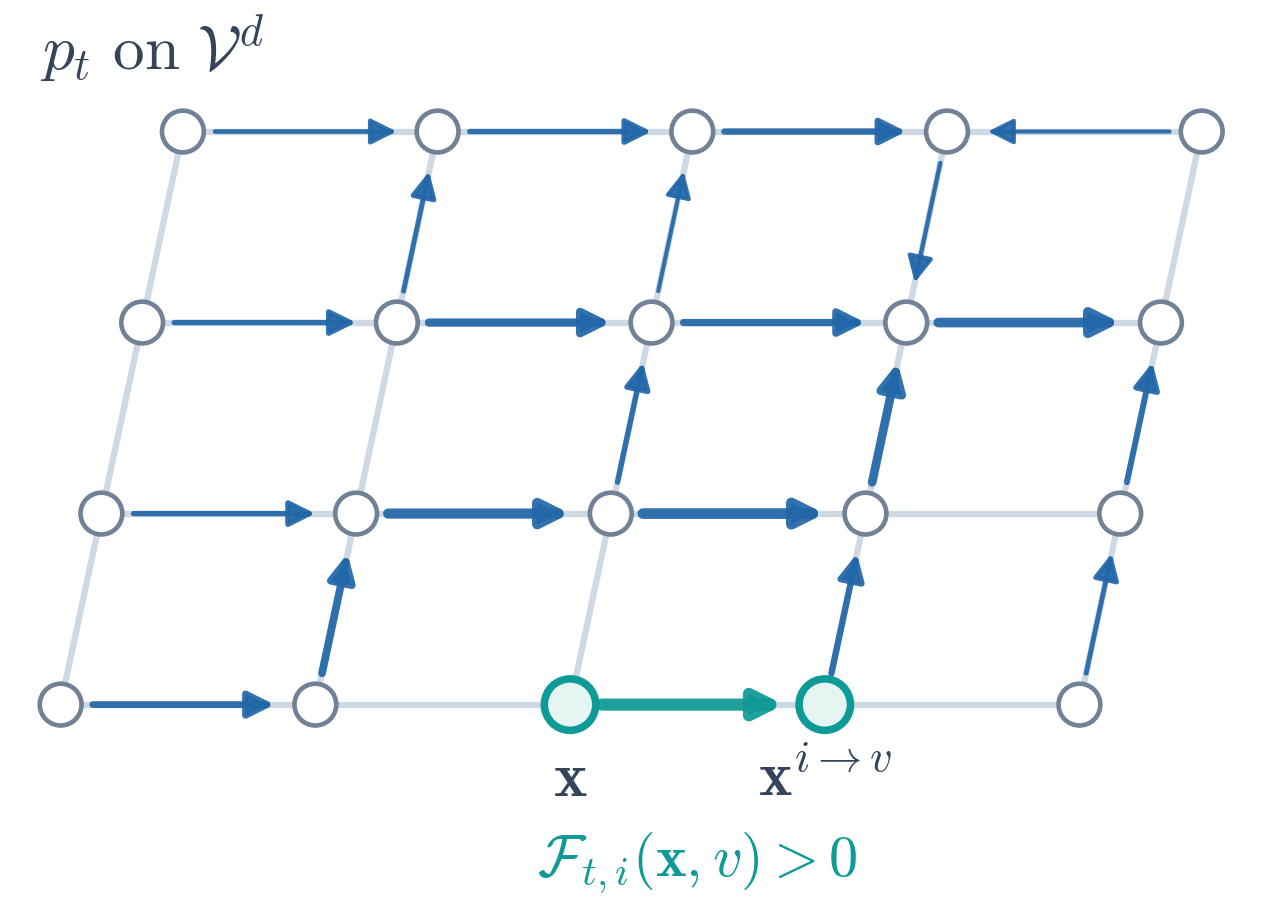}
    \caption{\textbf{Discrete probability flow.}
    Each node represents a state in $\mathcal V^d$; the shown links
    change one coordinate. The highlighted arrow depicts
    $\mathcal F_{t,i}(\mathbf x,v)>0$ from $\mathbf x$ to
    $\mathbf x^{i\to v}$. Other arrows show local probability
    transfer.}
    \label{fig:dfm}
\end{figure}

One realization of the conditional path draws independent switch times
$U_i\sim\mathrm{Unif}(0,1)$ and sets $X_t^i=X_0^i$ for $t<U_i$ and
$X_t^i=X_1^i$ for $t\geq U_i$. A coordinate with distinct endpoint
values makes a single jump. Given that it has not jumped by time $t$,
its instantaneous jump rate is $1/(1-t)$.

For a state $\mathbf x$ and a value $v\ne x_i$, write
$\mathbf x^{i\to v}$ for the state obtained by replacing $x_i$ with
$v$. The endpoint posterior and the corresponding marginal jump rate
are, for $0<t<1$ and $p_t(\mathbf x)>0$,
\begin{equation}
\label{eq:app-dfm-forward-rate}
q_{a,t,i}(v\mid\mathbf x)
:=\mathbb P(X_a^i=v\mid\mathbf X_t=\mathbf x),\quad a\in\{0,1\},
\qquad
r_{t,i}(\mathbf x,v)
=\frac{q_{1,t,i}(v\mid\mathbf x)}{1-t}.
\end{equation}
The rate averages the conditional jumps over endpoint pairs given the
complete current state. At an interior time, a jump from $\mathbf x$
to $\mathbf x^{i\to v}$ occurs with probability
$h\,r_{t,i}(\mathbf x,v)+o(h)$ over an interval of length $h$.
The CTMC with these rates has the prescribed marginals $p_t$, which
satisfy the forward equation
\begin{equation}
\label{eq:app-dfm-master}
\partial_t p_t(\mathbf x)
=\sum_{i=1}^d\sum_{v\ne x_i}
\left[
 p_t(\mathbf x^{i\to v})
 r_{t,i}(\mathbf x^{i\to v},x_i)
 -p_t(\mathbf x)r_{t,i}(\mathbf x,v)
\right].
\end{equation}
Each bracket is the incoming rate of probability from a neighboring
state minus the outgoing rate to that state.

To learn the rates, DFM predicts the target-coordinate posteriors from
$(t,\mathbf X_t)$ with a categorical model
$\widehat q^\theta_{1,t,i}(\cdot\mid\mathbf X_t)$. For
$t\sim\mathrm{Unif}(0,1)$, its cross-entropy objective is
\begin{equation}
\label{eq:app-dfm-training}
\mathcal L_{\mathrm{DFM}}(\theta)
=-\mathbb E_{t,\mathbf X_0,\mathbf X_1,\mathbf X_t}
\left[\sum_{i=1}^d
\log\widehat q^\theta_{1,t,i}(X_1^i\mid\mathbf X_t)\right],
\end{equation}
where $\mathbf X_t$ is drawn from the conditional bridge in
\eqref{eq:app-dfm-marginal-path}. Its population minimizer is
$q_{1,t,i}$ for each $(t,\mathbf x)$ on the path. These training
states can be sampled directly from the endpoint pair. To generate
from the fitted path, initialize at $P_0$ and simulate the CTMC using
$\widehat r_{t,i}(\mathbf x,v)
=\widehat q^\theta_{1,t,i}(v\mid\mathbf x)/(1-t)$.

\paragraph{Local probability flux.}
For $\mathbf y=\mathbf x^{i\to v}$, the outgoing probability per
unit time along the edge is $p_t(\mathbf x)r_{t,i}(\mathbf x,v)$;
the reverse flow is $p_t(\mathbf y)r_{t,i}(\mathbf y,x_i)$.
Their signed difference is
\begin{equation}
\label{eq:app-dfm-posterior-flux}
\begin{aligned}
\mathcal F_{t,i}(\mathbf x,v)
&=p_t(\mathbf x)r_{t,i}(\mathbf x,v)
 -p_t(\mathbf y)r_{t,i}(\mathbf y,x_i)=p_t(\mathbf x)\left[
 \frac{q_{1,t,i}(v\mid\mathbf x)}{1-t}
 -\frac{q_{0,t,i}(v\mid\mathbf x)}{t}
 \right].
\end{aligned}
\end{equation}
For the second equality, the source-posterior term describes the
reverse flow: when $X_0^i=v$ and $X_t^i=x_i$, coordinate $i$ was
copied from $X_1^i=x_i$, an event with factor $t$. The same endpoint
pair contributes to the forward rate at $\mathbf y$ with factor
$1-t$. Hence $\mathcal F_{t,i}(\mathbf x,v)
=-\mathcal F_{t,i}(\mathbf y,x_i)$, and
$\partial_t p_t(\mathbf x)
=-\sum_i\sum_{v\ne x_i}\mathcal F_{t,i}(\mathbf x,v)$.
Figure~\ref{fig:dfm} illustrates these flows along single-coordinate
edges.

At the midpoint, the two posterior contributions have the same
coefficient:
\begin{equation}
\label{eq:app-dfm-midpoint-flux}
\mathcal F_{0.5,i}(\mathbf x,v)
=2p_{0.5}(\mathbf x)
\left[q_{1,0.5,i}(v\mid\mathbf x)
      -q_{0,0.5,i}(v\mid\mathbf x)\right].
\end{equation}
Zero Flux estimates both endpoint posteriors at $t=0.5$ and uses
their difference to quantify local distributional change.

\section{Continuity Equation for the Coordinate-Wise Bridge}
\label{proof:ce}

Recall
\[
    \mathbf X_0 \sim P_0,
    \qquad
    \mathbf X_1 \sim P_1,
    \qquad
    \mathbf X_0 \perp \mathbf X_1.
\]
For \(t\in[0,1]\), define the coordinate-wise bridge \(\mathbf X_t\) as follows:
for each coordinate \(j\in[d]\),
\[
    X_t^j =
    \begin{cases}
        X_0^j, & \text{with probability } 1-t,\\
        X_1^j, & \text{with probability } t,
    \end{cases}
\]
independently over coordinates conditional on \((\mathbf X_0,\mathbf X_1)\).

For a fixed state \(\mathbf x=(x_1,\dots,x_d)\in\mathcal V^d\), recall
\[
    \mathbb P(X_t^j=x_j\mid \mathbf X _0,\mathbf X _1)
    :=
    (1-t)\mathbf 1\{X_0^j=x_j\}
    +
    t\mathbf 1\{X_1^j=x_j\}.
\]
Then
\[
    p_t(\mathbf x)
    :=
    \mathbb P(\mathbf X_t=\mathbf x).
\]

For \(b\neq x_i\), define the coordinate-wise posterior flux
\[
    \mathcal F_{t,i}(\mathbf x,b)
    :=
    \frac{
        \mathbb P(X_1^i=b,\mathbf X_t=\mathbf x)
    }{
        1-t
    }
    -
    \frac{
        \mathbb P(X_0^i=b,\mathbf X_t=\mathbf x)
    }{
        t
    },
    \qquad t\in(0,1).
\]

\begin{lemma}[Derivative of the path marginal]
\label{lem:path-marginal-derivative}
For every \(\mathbf x\in\mathcal V^d\) and \(t\in(0,1)\),
\[
    \partial_t p_t(\mathbf x)
    =
    \sum_{i=1}^d
    \mathbb E
    \left[
        \left(
            \mathbf 1\{X_1^i=x_i\}
            -
            \mathbf 1\{X_0^i=x_i\}
        \right)
        \prod_{j\neq i} \mathbb P(X_t^j=x_j\mid \mathbf X _0,\mathbf X _1)
    \right].
\]
\end{lemma}

\begin{proof}
Conditional on \((\mathbf X_0,\mathbf X_1)\), the coordinate-wise bridge
decisions are independent. Hence
\[
    \mathbb P(\mathbf X_t=\mathbf x\mid \mathbf X_0,\mathbf X_1)
    =
    \prod_{j=1}^d \mathbb P(X_t^j=x_j\mid \mathbf X _0,\mathbf X _1).
\]
Therefore, by the tower property,
\[
    p_t(\mathbf x)
    =
    \mathbb E
    \left[
        \prod_{j=1}^d \mathbb P(X_t^j=x_j\mid \mathbf X _0,\mathbf X _1)
    \right].
\]
Since the state space is finite, we may differentiate inside the expectation.
By the product rule,
\[
\begin{aligned}
    \partial_t p_t(\mathbf x)
    &=
    \mathbb E
    \left[
        \partial_t
        \prod_{j=1}^d \mathbb P(X_t^j=x_j\mid \mathbf X _0,\mathbf X _1)
    \right] \\
    &=
    \mathbb E
    \left[
        \sum_{i=1}^d
        \partial_t \mathbb P(X_t^i=x_i\mid \mathbf X _0,\mathbf X _1)
        \prod_{j\neq i}\mathbb P(X_t^j=x_j\mid \mathbf X _0,\mathbf X _1)
    \right].
\end{aligned}
\]
Moreover,
\[
    \partial_t \mathbb P(X_t^i=x_i\mid \mathbf X _0,\mathbf X _1)
    =
    \mathbf 1\{X_1^i=x_i\}
    -
    \mathbf 1\{X_0^i=x_i\}.
\]
Substituting this into the previous display gives
\[
    \partial_t p_t(\mathbf x)
    =
    \sum_{i=1}^d
    \mathbb E
    \left[
        \left(
            \mathbf 1\{X_1^i=x_i\}
            -
            \mathbf 1\{X_0^i=x_i\}
        \right)
        \prod_{j\neq i}\mathbb P(X_t^j=x_j\mid \mathbf X _0,\mathbf X _1)
    \right].
\]
\end{proof}

\begin{lemma}[Scaled endpoint probabilities]
\label{lem:scaled-endpoint-probabilities}
Fix \(i\in[d]\), \(\mathbf x\in\mathcal V^d\), and \(b\neq x_i\). Then
\[
    \frac{
        \mathbb P(X_1^i=b,\mathbf X_t=\mathbf x)
    }{
        1-t
    }
    =
    \mathbb E
    \left[
        \mathbf 1\{X_1^i=b\}
        \mathbf 1\{X_0^i=x_i\}
        \prod_{j\neq i}\mathbb P(X_t^j=x_j\mid \mathbf X _0,\mathbf X _1)
    \right],
\]
and
\[
    \frac{
        \mathbb P(X_0^i=b,\mathbf X_t=\mathbf x)
    }{
        t
    }
    =
    \mathbb E
    \left[
        \mathbf 1\{X_0^i=b\}
        \mathbf 1\{X_1^i=x_i\}
        \prod_{j\neq i}\mathbb P(X_t^j=x_j\mid \mathbf X _0,\mathbf X _1)
    \right].
\]
\end{lemma}

\begin{proof}
We prove the first identity. By the tower property,
\[
\begin{aligned}
    \mathbb P(X_1^i=b,\mathbf X_t=\mathbf x)
    &=
    \mathbb E
    \left[
        \mathbf 1\{X_1^i=b\}
        \mathbb P(\mathbf X_t=\mathbf x\mid \mathbf X_0,\mathbf X_1)
    \right] \\
    &=
    \mathbb E
    \left[
        \mathbf 1\{X_1^i=b\}
        \prod_{j=1}^d \mathbb P(X_t^j=x_j\mid \mathbf X _0,\mathbf X _1)
    \right] \\
    &=
    \mathbb E
    \left[
        \mathbf 1\{X_1^i=b\}
        \mathbb P(X_t^i=x_i\mid \mathbf X _0,\mathbf X _1)
        \prod_{j\neq i}\mathbb P(X_t^j=x_j\mid \mathbf X _0,\mathbf X _1)
    \right].
\end{aligned}
\]
Now expand the \(i\)-th factor:
\[
\begin{aligned}
    \mathbf 1\{X_1^i=b\}\mathbb P(X_t^i=x_i\mid \mathbf X _0,\mathbf X _1)
    &=
    \mathbf 1\{X_1^i=b\}
    \left[
        (1-t)\mathbf 1\{X_0^i=x_i\}
        +
        t\mathbf 1\{X_1^i=x_i\}
    \right] \\
    &=
    (1-t)
    \mathbf 1\{X_1^i=b\}
    \mathbf 1\{X_0^i=x_i\}
    +
    t
    \mathbf 1\{X_1^i=b\}
    \mathbf 1\{X_1^i=x_i\}.
\end{aligned}
\]
Since \(b\neq x_i\),
\[
    \mathbf 1\{X_1^i=b\}\mathbf 1\{X_1^i=x_i\}=0.
\]
Therefore,
\[
    \mathbf 1\{X_1^i=b\}\mathbb P(X_t^i=x_i\mid \mathbf X _0,\mathbf X _1)
    =
    (1-t)
    \mathbf 1\{X_1^i=b\}
    \mathbf 1\{X_0^i=x_i\}.
\]
Substituting this into the expectation gives
\[
    \mathbb P(X_1^i=b,\mathbf X_t=\mathbf x)
    =
    (1-t)
    \mathbb E
    \left[
        \mathbf 1\{X_1^i=b\}
        \mathbf 1\{X_0^i=x_i\}
        \prod_{j\neq i}\mathbb P(X_t^j=x_j\mid \mathbf X _0,\mathbf X _1)
    \right].
\]
Dividing by \(1-t\) proves the first identity.

The second identity is analogous. Indeed,
\[
\begin{aligned}
    \mathbb P(X_0^i=b,\mathbf X_t=\mathbf x)
    &=
    \mathbb E
    \left[
        \mathbf 1\{X_0^i=b\}
        \mathbb P(X_t^i=x_i\mid \mathbf X _0,\mathbf X _1)
        \prod_{j\neq i}\mathbb P(X_t^j=x_j\mid \mathbf X _0,\mathbf X _1)
    \right].
\end{aligned}
\]
Expanding \(\mathbb P(X_t^i=x_i\mid \mathbf X _0,\mathbf X _1)\), we get
\[
\begin{aligned}
    \mathbf 1\{X_0^i=b\}\mathbb P(X_t^i=x_i\mid \mathbf X _0,\mathbf X _1)
    &=\!
    (1-t)
    \mathbf 1\{X_0^i=b\}
    \mathbf 1\{X_0^i=x_i\}
    +
    t
    \mathbf 1\{X_0^i=b\}
    \mathbf 1\{X_1^i=x_i\}.
\end{aligned}
\]
Since \(b\neq x_i\),
\[
    \mathbf 1\{X_0^i=b\}\mathbf 1\{X_0^i=x_i\}=0.
\]
Thus
\[
    \mathbf 1\{X_0^i=b\}\mathbb P(X_t^i=x_i\mid \mathbf X _0,\mathbf X _1)
    =
    t
    \mathbf 1\{X_0^i=b\}
    \mathbf 1\{X_1^i=x_i\}.
\]
Hence
\[
    \mathbb P(X_0^i=b,\mathbf X_t=\mathbf x)
    =
    t
    \mathbb E
    \left[
        \mathbf 1\{X_0^i=b\}
        \mathbf 1\{X_1^i=x_i\}
        \prod_{j\neq i}\mathbb P(X_t^j=x_j\mid \mathbf X _0,\mathbf X _1)
    \right],
\]
and division by \(t\) proves the second identity.
\end{proof}

\begin{lemma}[Flux-indicator identity]
\label{lem:flux-indicator-identity}
For every \(i\in[d]\), \(\mathbf x\in\mathcal V^d\), and \(t\in(0,1)\),
\[
    \sum_{b\neq x_i}\mathcal F_{t,i}(\mathbf x,b)
    =
    -
    \mathbb E
    \left[
        \left(
            \mathbf 1\{X_1^i=x_i\}
            -
            \mathbf 1\{X_0^i=x_i\}
        \right)
        \prod_{j\neq i}\mathbb P(X_t^j=x_j\mid \mathbf X _0,\mathbf X _1)
    \right].
\]
\end{lemma}

\begin{proof}
By definition,
\[
    \mathcal F_{t,i}(\mathbf x,b)
    =
    \frac{
        \mathbb P(X_1^i=b,\mathbf X_t=\mathbf x)
    }{
        1-t
    }
    -
    \frac{
        \mathbb P(X_0^i=b,\mathbf X_t=\mathbf x)
    }{
        t
    }.
\]
Using Lemma~\ref{lem:scaled-endpoint-probabilities}, we obtain
\[
\begin{aligned}
    \sum_{b\neq x_i}\mathcal F_{t,i}(\mathbf x,b)
    &=
    \sum_{b\neq x_i}
    \mathbb E
    \left[
        \mathbf 1\{X_1^i=b\}
        \mathbf 1\{X_0^i=x_i\}
        \prod_{j\neq i}\mathbb P(X_t^j=x_j\mid \mathbf X _0,\mathbf X _1)
    \right] \\
    &\quad
    -
    \sum_{b\neq x_i}
    \mathbb E
    \left[
        \mathbf 1\{X_0^i=b\}
        \mathbf 1\{X_1^i=x_i\}
        \prod_{j\neq i}\mathbb P(X_t^j=x_j\mid \mathbf X _0,\mathbf X _1)
    \right].
\end{aligned}
\]
Since the alphabet \(\mathcal V\) is finite, we may exchange the finite sum and
expectation. Thus
\[
\begin{aligned}
    \sum_{b\neq x_i}\mathcal F_{t,i}(\mathbf x,b)
    &=
    \mathbb E
    \left[
    \begin{aligned}
        &
        \left(
            \mathbf 1\{X_0^i=x_i\}
            \sum_{b\neq x_i}\mathbf 1\{X_1^i=b\}
            -
            \mathbf 1\{X_1^i=x_i\}
            \sum_{b\neq x_i}\mathbf 1\{X_0^i=b\}
        \right)
        \\
        &\quad{}\times
        \prod_{j\neq i}\mathbb P(X_t^j=x_j\mid \mathbf X _0,\mathbf X _1)
    \end{aligned}
    \right] \\
    &=
    \mathbb E
    \left[
    \begin{aligned}
        &
        \left(
            \mathbf 1\{X_0^i=x_i\}
            \mathbf 1\{X_1^i\neq x_i\}
            -
            \mathbf 1\{X_1^i=x_i\}
            \mathbf 1\{X_0^i\neq x_i\}
        \right)
        \\
        &\quad{}\times
        \prod_{j\neq i}\mathbb P(X_t^j=x_j\mid \mathbf X _0,\mathbf X _1)
    \end{aligned}
    \right].
\end{aligned}
\]
Using the indicator identity
\[
    \mathbf 1\{X_0^i=x_i\}\mathbf 1\{X_1^i\neq x_i\}
    -
    \mathbf 1\{X_0^i\neq x_i\}\mathbf 1\{X_1^i=x_i\}
    =
    \mathbf 1\{X_0^i=x_i\}
    -
    \mathbf 1\{X_1^i=x_i\},
\]
we get
\[
\begin{aligned}
    \sum_{b\neq x_i}\mathcal F_{t,i}(\mathbf x,b)
    &=
    \mathbb E
    \left[
        \left(
            \mathbf 1\{X_0^i=x_i\}
            -
            \mathbf 1\{X_1^i=x_i\}
        \right)
        \prod_{j\neq i}\mathbb P(X_t^j=x_j\mid \mathbf X _0,\mathbf X _1)
    \right] \\
    &=
    -
    \mathbb E
    \left[
        \left(
            \mathbf 1\{X_1^i=x_i\}
            -
            \mathbf 1\{X_0^i=x_i\}
        \right)
        \prod_{j\neq i}\mathbb P(X_t^j=x_j\mid \mathbf X _0,\mathbf X _1)
    \right].
\end{aligned}
\]
This proves the claim.
\end{proof}

\begin{theorem}[Continuity equation]
For every \(\mathbf x\in\mathcal V^d\) and every \(t\in(0,1)\),
\[
    \partial_t p_t(\mathbf x)
    =
    -
    \sum_{i=1}^d
    \sum_{b\neq x_i}
    \mathcal F_{t,i}(\mathbf x,b).
\]
\end{theorem}
\begin{proof}
By Lemma~\ref{lem:path-marginal-derivative},
\[
    \partial_t p_t(\mathbf x)
    =
    \sum_{i=1}^d
    \mathbb E
    \left[
        \left(
            \mathbf 1\{X_1^i=x_i\}
            -
            \mathbf 1\{X_0^i=x_i\}
        \right)
        \prod_{j\neq i}\mathbb P(X_t^j=x_j\mid \mathbf X _0,\mathbf X _1)
    \right].
\]
By Lemma~\ref{lem:flux-indicator-identity}, for each coordinate \(i\),
\[
    \mathbb E
    \left[
        \left(
            \mathbf 1\{X_1^i=x_i\}
            -
            \mathbf 1\{X_0^i=x_i\}
        \right)
        \prod_{j\neq i}\mathbb P(X_t^j=x_j\mid \mathbf X _0,\mathbf X _1)
    \right]
    =
    -
    \sum_{b\neq x_i}
    \mathcal F_{t,i}(\mathbf x,b).
\]
Substituting this identity into the expression for
\(\partial_t p_t(\mathbf x)\) gives
\[
    \partial_t p_t(\mathbf x)
    =
    -
    \sum_{i=1}^d
    \sum_{b\neq x_i}
    \mathcal F_{t,i}(\mathbf x,b).
\]
This completes the proof.
\end{proof}
\paragraph{Proof of Proposition~\ref{prop:edge_antisymmetry}.}
Let $\mathbf y=\mathbf x^{i\to v}$, so that $y^i=v$ and
$y^j=x^j$ for every $j\neq i$. By the same bridge-factorization
argument used in Lemma~\ref{lem:flux-indicator-identity},
\begin{align}
F_{t,i}(\mathbf x,v)
=
\mathbb E\Bigg[
&\Big(
\mathbf 1\{X_1^i=v\}\mathbf 1\{X_0^i=x^i\}
-
\mathbf 1\{X_0^i=v\}\mathbf 1\{X_1^i=x^i\}
\Big)
\nonumber\\
&\qquad\qquad\times
\prod_{j\neq i}
\mathbb P(X_t^j=x^j\mid \mathbf X_0,\mathbf X_1)
\Bigg].
\label{eq:edge_flux_indicator}
\end{align}
Since $\mathbf x$ and $\mathbf y$ differ only at coordinate $i$,
the product over $j\neq i$ is unchanged when the same edge is viewed
from $\mathbf y$. Therefore,
\begin{align}
F_{t,i}(\mathbf y,x^i)
=
\mathbb E\Bigg[
&\Big(
\mathbf 1\{X_1^i=x^i\}\mathbf 1\{X_0^i=v\}
-
\mathbf 1\{X_0^i=x^i\}\mathbf 1\{X_1^i=v\}
\Big)
\nonumber\\
&\qquad\qquad\times
\prod_{j\neq i}
\mathbb P(X_t^j=x^j\mid \mathbf X_0,\mathbf X_1)
\Bigg]
=
-F_{t,i}(\mathbf x,v).
\end{align}
This proves the claim.
\section{Midpoint Edge Zero Flux Criterion}
\label{proof:mid}

Recall
\[
    \mathbf X_0 \sim P_0,
    \qquad
    \mathbf X_1 \sim P_1,
    \qquad
    \mathbf X_0 \perp \mathbf X_1.
\]
Let \(\mathbf Z\) be generated by the coordinate-wise midpoint bridge:
for each coordinate \(j\in[d]\),
\[
    Z^j =
    \begin{cases}
        X_0^j, & \text{with probability } 0.5,\\
        X_1^j, & \text{with probability } 0.5,
    \end{cases}
\]
independently over coordinates conditional on \((\mathbf X_0,\mathbf X_1)\).

For \(\mathbf x\in\mathcal V^d\), \(i\in[d]\), and \(b\neq x_i\), recall
the midpoint edge flux
\[
    \mathcal F_{0.5,i}(\mathbf x,b)
    :=
    2\,\mathbb P(X_1^i=b,\mathbf Z=\mathbf x)
    -
    2\,\mathbb P(X_0^i=b,\mathbf Z=\mathbf x).
\]

\begin{lemma}[Posterior coordinate symmetry]
\label{lem:posterior-coordinate-symmetry}
Assume that
\[
    \mathcal F_{0.5,i}(\mathbf x,b)=0
    \qquad
    \forall\, \mathbf x\in\mathcal V^d,\ i\in[d],\ b\neq x_i.
\]
Then, for every \(i\in[d]\), every test function
\(\phi:\mathcal V^d\to\mathbb R\) and every
\(h:\mathcal V\to\mathbb R\),
\[
    \mathbb E\left[
        \phi(\mathbf Z)
        \bigl\{
            h(X_1^i)-h(X_0^i)
        \bigr\}
    \right]
    =
    0.
\]
\end{lemma}

\begin{proof}
Fix \(\mathbf x\in\mathcal V^d\), \(i\in[d]\), and \(b\neq x_i\).
By the definition of the midpoint edge flux, the condition
\(\mathcal F_{0.5,i}(\mathbf x,b)=0\) implies
\[
    \mathbb P(X_1^i=b,\mathbf Z=\mathbf x)
    =
    \mathbb P(X_0^i=b,\mathbf Z=\mathbf x)
    \qquad
    \forall\, b\neq x_i.
\]
Since
\[
    \sum_{b\in\mathcal V}
    \mathbb P(X_1^i=b,\mathbf Z=\mathbf x)
    =
    \mathbb P(\mathbf Z=\mathbf x)
    =
    \sum_{b\in\mathcal V}
    \mathbb P(X_0^i=b,\mathbf Z=\mathbf x),
\]
the equality must also hold for \(b=x_i\). Therefore,
\[
    \mathbb P(X_1^i=b,\mathbf Z=\mathbf x)
    =
    \mathbb P(X_0^i=b,\mathbf Z=\mathbf x)
    \qquad
    \forall\, b\in\mathcal V.
\]
Thus, for arbitrary \(\phi:\mathcal V^d\to\mathbb R\) and
\(h:\mathcal V\to\mathbb R\),
\begin{align}
    \mathbb E[\phi(\mathbf Z)h(X_1^i)]
    &=
    \sum_{\mathbf x\in\mathcal V^d}
    \sum_{b\in\mathcal V}
    \phi(\mathbf x)h(b)
    \mathbb P(X_1^i=b,\mathbf Z=\mathbf x) \nonumber\\
    &=
    \sum_{\mathbf x\in\mathcal V^d}
    \sum_{b\in\mathcal V}
    \phi(\mathbf x)h(b)
    \mathbb P(X_0^i=b,\mathbf Z=\mathbf x) =
    \mathbb E[\phi(\mathbf Z)h(X_0^i)].
\end{align}
Hence
\[
    \mathbb E\left[
        \phi(\mathbf Z)
        \bigl\{
            h(X_1^i)-h(X_0^i)
        \bigr\}
    \right]
    =
    0.
\]
\end{proof}

\begin{lemma}[Product moment equivalence]
\label{lem:product-moment-equivalence}
Assume that for every \(i\in[d]\), every 
\(\phi:\mathcal V^d\to\mathbb R\), and every 
\(h:\mathcal V\to\mathbb R\),
\[
    \mathbb E\left[
        \phi(\mathbf Z)
        \{h(X_1^i)-h(X_0^i)\}
    \right]=0.
\]
Then, for every \(U\subseteq[d]\) and every collection of test functions
\(\{h_j:\mathcal V\to\mathbb R\}_{j\in U}\),
\[
    \mathbb E_{P_0}\left[
        \prod_{j\in U}h_j(X^j)
    \right]
    =
    \mathbb E_{P_1}\left[
        \prod_{j\in U}h_j(X^j)
    \right].
\]
\end{lemma}
\begin{proof}
We prove the claim by induction on the size of the subset.

For \(U=\emptyset\), the claim is immediate since both sides equal \(1\).
For \(|U|=1\), say \(U=\{i\}\), take \(\phi\equiv 1\) and \(h=h_i\). 
The assumed posterior coordinate symmetry gives
\[
    \mathbb E[h_i(X_1^i)]
    =
    \mathbb E[h_i(X_0^i)].
\]
Since \(\mathbf X_0\sim P_0\) and \(\mathbf X_1\sim P_1\), this is exactly
\[
    \mathbb E_{P_1}[h_i(X^i)]
    =
    \mathbb E_{P_0}[h_i(X^i)].
\]

Now assume that the claim holds for all subsets of size at most \(k\). 
Let \(U\subseteq[d]\) with \(|U|=k+1\), and let
\(\{h_j:\mathcal V\to\mathbb R\}_{j\in U}\) be arbitrary. 
Pick \(i\in U\), and set
\[
    A:=U\setminus\{i\}.
\]
For \(B\subseteq U\) and \(r\in\{0,1\}\), define
\[
    m_r(B)
    :=
    \mathbb E_{P_r}\left[
        \prod_{j\in B}h_j(X^j)
    \right],
    \qquad
    m_r(\emptyset):=1.
\]

Choose
\[
    \phi(\mathbf Z)
    =
    \prod_{j\in A}h_j(Z^j),
    \qquad
    h=h_i .
\]
By the assumed posterior coordinate symmetry,
\[
    0
    =
    \mathbb E\left[
        \prod_{j\in A}h_j(Z^j)
        \{h_i(X_1^i)-h_i(X_0^i)\}
    \right].
\]
Using the tower property and conditioning on \((\mathbf X_0,\mathbf X_1)\), we get
\[
\begin{aligned}
0
&=
\mathbb E\left[
    \mathbb E\left[
        \prod_{j\in A}h_j(Z^j)
        \{h_i(X_1^i)-h_i(X_0^i)\}
        \,\middle|\,
        \mathbf X_0,\mathbf X_1
    \right]
\right].
\end{aligned}
\]
Conditional on \((\mathbf X_0,\mathbf X_1)\), the quantity
\(\{h_i(X_1^i)-h_i(X_0^i)\}\) is fixed. Therefore
\[
\begin{aligned}
0
=
\mathbb E\left[
    \{h_i(X_1^i)-h_i(X_0^i)\}
    \,
    \mathbb E\left[
        \prod_{j\in A}h_j(Z^j)
        \,\middle|\,
        \mathbf X_0,\mathbf X_1
    \right]
\right].
\end{aligned}
\]

We now compute the inner conditional expectation. Given
\((\mathbf X_0,\mathbf X_1)\), the midpoint bridge chooses each coordinate
independently from \(\mathbf X_0\) or \(\mathbf X_1\) with probability
\(0.5\). Hence
\[
\begin{aligned}
\mathbb E\left[
    \prod_{j\in A}h_j(Z^j)
    \,\middle|\,
    \mathbf X_0,\mathbf X_1
\right]
&=
\prod_{j\in A}
\mathbb E\left[
    h_j(Z^j)
    \,\middle|\,
    \mathbf X_0,\mathbf X_1
\right]
\\
&=
\prod_{j\in A}
\left[
    \frac12 h_j(X_1^j)
    +
    \frac12 h_j(X_0^j)
\right]=
2^{-|A|}
\prod_{j\in A}
\left[
    h_j(X_1^j)+h_j(X_0^j)
\right].
\end{aligned}
\]
Expanding the product over subsets \(T\subseteq A\), where \(T\) indexes the coordinates chosen from \(\mathbf X_1\), gives\footnote{
For example, if \(A=\{1,2,3\}\), then the expansion of
$\prod_{j\in A}\big(h_j(X_1^j)+h_j(X_0^j)\big)$
has one term for each subset \(T\subseteq A\). The subset \(T\) records the
coordinates whose factors contribute the \(\mathbf X_1\)-term. For instance,
\[
\begin{array}{c|c}
T & 
\left(\prod_{j\in T}h_j(X_1^j)\right)
\left(\prod_{j\in A\setminus T}h_j(X_0^j)\right)
\\
\hline
\{1,2,3\} 
& (h_1(X_1^1)h_2(X_1^2)h_3(X_1^3))()
\\
\{1,2\} 
& (h_1(X_1^1)h_2(X_1^2))(h_3(X_0^3))
\\
\{1,3\} 
& (h_1(X_1^1)h_3(X_1^3))(h_2(X_0^2))
\\
\{2,3\} 
& (h_2(X_1^2)h_3(X_1^3))(h_1(X_0^1))
\\
\{1\} 
& (h_1(X_1^1))(h_2(X_0^2)h_3(X_0^3))
\\
\{2\} 
& (h_2(X_1^2))(h_1(X_0^1)h_3(X_0^3))
\\
\{3\} 
& (h_3(X_1^3))(h_1(X_0^1)h_2(X_0^2))
\\
\emptyset 
& ()(h_1(X_0^1)h_2(X_0^2)h_3(X_0^3))
\end{array}
\]
}
\[
\begin{aligned}
\mathbb E\left[
    \prod_{j\in A}h_j(Z^j)
    \,\middle|\,
    \mathbf X_0,\mathbf X_1
\right]
=
2^{-|A|}
\sum_{T\subseteq A}
\left(
    \prod_{j\in T}h_j(X_1^j)
\right)
\left(
    \prod_{j\in A\setminus T}h_j(X_0^j)
\right).
\end{aligned}
\]
Substituting this expression into the outer expectation yields
\[
\begin{aligned}
0
=
2^{-|A|}
\sum_{T\subseteq A}
\mathbb E\Bigg[
&
\{h_i(X_1^i)-h_i(X_0^i)\}
\left(
    \prod_{j\in T}h_j(X_1^j)
\right)
\left(
    \prod_{j\in A\setminus T}h_j(X_0^j)
\right)
\Bigg].
\end{aligned}
\]
Distributing the difference gives
\[
\begin{aligned}
0
=
2^{-|A|}
\sum_{T\subseteq A}
\Bigg\{
&
\mathbb E\left[
    h_i(X_1^i)
    \prod_{j\in T}h_j(X_1^j)
    \prod_{j\in A\setminus T}h_j(X_0^j)
\right] \\
&{}-\mathbb E\left[
    h_i(X_0^i)
    \prod_{j\in T}h_j(X_1^j)
    \prod_{j\in A\setminus T}h_j(X_0^j)
\right]
\Bigg\}.
\end{aligned}
\]
Since \(\mathbf X_0\perp \mathbf X_1\), functions of \(\mathbf X_0\) are independent of functions of \(\mathbf X_1\). Therefore
\begin{align}
0
=
2^{-|A|}
\sum_{T\subseteq A}
\Bigg\{
&
\mathbb E_{P_1}\left[
    h_i(X^i)
    \prod_{j\in T}h_j(X^j)
\right]
\mathbb E_{P_0}\left[
    \prod_{j\in A\setminus T}h_j(X^j)
\right]\nonumber
\\
&-
\mathbb E_{P_1}\left[
    \prod_{j\in T}h_j(X^j)
\right]
\mathbb E_{P_0}\left[
    h_i(X^i)
    \prod_{j\in A\setminus T}h_j(X^j)
\right]
\Bigg\}.
\end{align}
Using the notation \(m_0,m_1\), this becomes
\begin{align}
0
=
2^{-|A|}
\sum_{T\subseteq A}
\left\{
    m_1(T\cup\{i\})m_0(A\setminus T)
    -
    m_1(T)m_0((A\setminus T)\cup\{i\})
\right\}.
\end{align}
By the induction hypothesis, for every subset \(B\subseteq[d]\) with
\(|B|\le k\), we have
\[
    m_0(B)=m_1(B).
\]
We now examine the sum
\[
\sum_{T\subseteq A}
\left\{
    m_1(T\cup\{i\})m_0(A\setminus T)
    -
    m_1(T)m_0((A\setminus T)\cup\{i\})
\right\}.
\]

First consider a non-boundary subset
\[
    \emptyset\neq T\neq A .
\]
Then
\[
    |T\cup\{i\}|\le k,
    \qquad
    |A\setminus T|\le k,
    \qquad
    |T|\le k,
    \qquad
    |(A\setminus T)\cup\{i\}|\le k .
\]
Hence the induction hypothesis applies to all four moments appearing in the
term indexed by \(T\). Thus the \(T\)-term equals
\[
    m_0(T\cup\{i\})m_0(A\setminus T)
    -
    m_0(T)m_0((A\setminus T)\cup\{i\}).
\]
Now look at the term indexed by \(A\setminus T\). Again by the induction
hypothesis, this term equals
\[
    m_0((A\setminus T)\cup\{i\})m_0(T)
    -
    m_0(A\setminus T)m_0(T\cup\{i\}).
\]
This is exactly the negative of the previous display. Therefore the terms
indexed by \(T\) and \(A\setminus T\) cancel pairwise.

It remains to consider the two boundary cases \(T=A\) and \(T=\emptyset\).
For \(T=A\), the corresponding term is
\[
    m_1(A\cup\{i\})m_0(\emptyset)
    -
    m_1(A)m_0(\{i\})
    =
    m_1(U)-m_1(A)m_0(\{i\}),
\]
since \(U=A\cup\{i\}\) and \(m_0(\emptyset)=1\). For \(T=\emptyset\), the
corresponding term is
\[
    m_1(\{i\})m_0(A)
    -
    m_1(\emptyset)m_0(A\cup\{i\})
    =
    m_1(\{i\})m_0(A)-m_0(U).
\]
Adding the two boundary terms gives
\[
    m_1(U)-m_0(U)
    -
    m_1(A)m_0(\{i\})
    +
    m_1(\{i\})m_0(A).
\]
By the induction hypothesis,
\[
    m_1(A)=m_0(A),
    \qquad
    m_1(\{i\})=m_0(\{i\}).
\]
Therefore the lower-order boundary terms cancel:
\[
    -m_1(A)m_0(\{i\})
    +
    m_1(\{i\})m_0(A)
    =
    -m_0(A)m_0(\{i\})
    +
    m_0(\{i\})m_0(A)
    =
    0.
\]
Thus the whole sum reduces to
\[
    m_1(U)-m_0(U).
\]
Since the original identity was
\[
    0
    =
    2^{-|A|}
    \sum_{T\subseteq A}
    \left\{
        m_1(T\cup\{i\})m_0(A\setminus T)
        -
        m_1(T)m_0((A\setminus T)\cup\{i\})
    \right\},
\]
we obtain
\[
    0=2^{-|A|}\{m_1(U)-m_0(U)\}.
\]
Hence
\[
    m_1(U)=m_0(U).
\]
This completes the induction.
\end{proof}

\begin{lemma}[Product moments identify the finite product distribution]
\label{lem:product-moments-identify}
Assume that for every subset \(S\subseteq[d]\) and every collection of test
functions \(\{h_j:\mathcal V\to\mathbb R\}_{j\in S}\),
\[
    \mathbb E_{P_0}
    \left[
        \prod_{j\in S} h_j(X^j)
    \right]
    =
    \mathbb E_{P_1}
    \left[
        \prod_{j\in S} h_j(X^j)
    \right].
\]
Then
\[
    P_0=P_1.
\]
\end{lemma}

\begin{proof}
Since \(\mathcal V^d\) is finite, it is enough to show that the two distributions agree
on every point \(\mathbf x=(x_1,\dots,x_d)\in\mathcal V^d\).
Choose
\[
    h_j(a):=\mathbf 1\{a=x_j\},
    \qquad
    j=1,\dots,d.
\]
Taking \(S=[d]\), the product moment equality gives
\begin{align}
    P_0(\mathbf x)
    &=
    \mathbb E_{P_0}
    \left[
        \mathbf 1\{X^1=x_1,...,X^d=x_d\}
    \right]=
    \mathbb E_{P_0}
    \left[
        \prod_{j=1}^d \mathbf 1\{X^j=x_j\}
    \right] \nonumber
    \\&=
    \mathbb E_{P_1}
    \left[
        \prod_{j=1}^d \mathbf 1\{X^j=x_j\}
    \right] 
    =
    \mathbb E_{P_1}
    \left[
        \mathbf 1\{X^1=x_1,...,X^d=x_d\}
    \right]=P_1(\mathbf x).
\end{align}
Since this holds for every \(\mathbf x\in\mathcal V^d\), we conclude that
\(P_0=P_1\).
\end{proof}

\subsection{Proof of Theorem~\ref{thm:midpoint}}
\begin{proof}
We prove both directions.
\paragraph{Forward Direction ($\Rightarrow$).}
Assume that
\[
    \mathcal F_{0.5,i}(\mathbf x,b)=0
    \qquad
    \forall\, \mathbf x\in\mathcal V^d,\ i\in[d],\ b\neq x_i.
\]
By Lemma~\ref{lem:posterior-coordinate-symmetry}, posterior coordinate
symmetry holds:
\[
    \mathbb E
    \left[
        \phi(\mathbf Z)
        \bigl\{
            h(X_1^i)-h(X_0^i)
        \bigr\}
    \right]
    =
    0
\]
for all \(\phi\), \(h\), and \(i\). By
Lemma~\ref{lem:product-moment-equivalence}, this implies that all product
moments of \(P_0\) and \(P_1\) agree. By
Lemma~\ref{lem:product-moments-identify}, equality of all product moments on
the finite product space \(\mathcal V^d\) implies
\[
    P_0=P_1.
\]

\paragraph{Reverse Direction ($\Leftarrow$).}
Assume that \(P_0=P_1\). Since \(\mathbf X_0\) and \(\mathbf X_1\) are
independent and identically distributed, the joint law of
\((\mathbf X_0,\mathbf X_1)\) is invariant under swapping the two endpoints.
Moreover, the coordinate-wise midpoint bridge is invariant under the same swap,
because each coordinate chooses either endpoint with probability \(0.5\).
Therefore, for every \(\mathbf x\in\mathcal V^d\), \(i\in[d]\), and
\(b\in\mathcal V\),
\[
    \mathbb P(X_1^i=b,\mathbf Z=\mathbf x)
    =
    \mathbb P(X_0^i=b,\mathbf Z=\mathbf x).
\]
Hence, for every \(b\neq x_i\),
\[
    \mathcal F_{0.5,i}(\mathbf x,b)
    =
    2\,\mathbb P(X_1^i=b,\mathbf Z=\mathbf x)
    -
    2\,\mathbb P(X_0^i=b,\mathbf Z=\mathbf x)
    =
    0.
\]
This proves the converse direction and completes the proof.
\end{proof}
\subsection{Proof of Proposition~\ref{prop:population-properties}}
\label{proof:population-properties}
\begin{proof}
Nonnegativity follows directly because \(D^\ast\) is an expectation of
squared fluxes. Moreover, conditional on any midpoint state
\(\mathbf Z=\mathbf z\), the bridge construction ensures that \(Z^i\) equals
at least one of \(X_0^i\) and \(X_1^i\). Hence
\begin{equation}
    \sum_{b\neq z_i}
    \left[
        q_{0,i}^\ast(b\mid\mathbf z)
        +
        q_{1,i}^\ast(b\mid\mathbf z)
    \right]
    \leq 1 .
\end{equation}
Therefore,
\begin{align}
    \sum_{b\neq z_i} F_i^\ast(\mathbf z,b)^2
    &=
    4\sum_{b\neq z_i}
    \left[
        q_{1,i}^\ast(b\mid\mathbf z)
        -
        q_{0,i}^\ast(b\mid\mathbf z)
    \right]^2                                                     \\
    &\leq
    4\left(
        \sum_{b\neq z_i}
        \left|
            q_{1,i}^\ast(b\mid\mathbf z)
            -
            q_{0,i}^\ast(b\mid\mathbf z)
        \right|
    \right)^2                                                     \leq
    4\left(
        \sum_{b\neq z_i}
        \left[
            q_{1,i}^\ast(b\mid\mathbf z)
            +
            q_{0,i}^\ast(b\mid\mathbf z)
        \right]
    \right)^2            \leq 4 .
\end{align}
Summing over coordinates and averaging over \(\mathbf Z\) yields
\(D^\ast(P_0,P_1)\leq 4d\).

To prove symmetry, swapping the endpoint labels maps the product coupling
\(P_0\times P_1\) to \(P_1\times P_0\). Complementing every midpoint
mask bit leaves the midpoint unchanged and does not change the mask
distribution. The midpoint law is therefore unchanged, while
\(q_{0,i}^\ast(\cdot\mid\mathbf z)\) and
\(q_{1,i}^\ast(\cdot\mid\mathbf z)\) are exchanged. Consequently, every
\(F_i^\ast(\mathbf z,b)\) changes sign, and its squared expectation is
unchanged. Thus
\[
    D^\ast(P_0,P_1)=D^\ast(P_1,P_0).
\]

For separation, suppose first that \(D^\ast=0\). Every squared normalized
flux then vanishes on the support of \(p_{0.5}\), and hence
\(\mathcal F_{0.5,i}=p_{0.5}F_i^\ast\) vanishes there. Off this support, both
joint probabilities in \eqref{eq:posterior-flux} are zero, so the midpoint
flux also vanishes. Theorem~\ref{thm:midpoint} therefore gives \(P_0=P_1\).
Conversely, if \(P_0=P_1\), endpoint-exchange symmetry makes every midpoint
flux, and hence every normalized flux on the midpoint support, equal to zero.
Thus \(D^\ast=0\).

Finally, independence cannot be replaced by an arbitrary coupling.
Let \(d=1\), let \(\mathcal V=\{0,1,2\}\), and let \(X_0\) be uniform on
\(\mathcal V\). Couple the endpoints by
\[
    X_1=X_0+1 \pmod 3 .
\]
Both endpoint marginals are uniform, so \(P_0=P_1\). However, conditional on
\(Z=z\), the endpoint pairs
\[
    (X_0,X_1)=(z,z+1)
    \quad\text{and}\quad
    (X_0,X_1)=(z-1,z)
\]
each occur with probability \(0.5\). It follows that
\[
    F_1^\ast(z,z+1)=1,
    \qquad
    F_1^\ast(z,z-1)=-1,
\]
where arithmetic is modulo \(3\). Therefore
\(D^\ast(P_0,P_1)=2\), despite \(P_0=P_1\). Thus a non-exchangeable coupling can
introduce directional flux unrelated to any difference between the two
endpoint marginals.
\end{proof}

\subsection{Proof of Proposition~\ref{prop:irrelevant-invariance}}
\label{app:irrelevant-invariance}

\begin{proof}
The product structure of the endpoint distributions and independent endpoint
coupling imply that
\[
    (\mathbf X_{0,S},\mathbf X_{1,S})
    \perp
    (\mathbf X_{0,S^c},\mathbf X_{1,S^c}).
\]
The coordinate-copy masks are independent across the two blocks.
Hence the corresponding endpoint--midpoint triples
\[
    (\mathbf X_{0,S},\mathbf X_{1,S},\mathbf Z_S)
    \quad\text{and}\quad
    (\mathbf X_{0,S^c},\mathbf X_{1,S^c},\mathbf Z_{S^c})
\]
are independent.

For \(i\in S\) and any supported midpoint \(\mathbf z\),
this independence gives
\[
\begin{aligned}
    \mathbb P(X_0^i=v\mid\mathbf Z=\mathbf z)
    &=
    \mathbb P(X_0^i=v\mid\mathbf Z_S=\mathbf z_S),\\
    \mathbb P(X_1^i=v\mid\mathbf Z=\mathbf z)
    &=
    \mathbb P(X_1^i=v\mid\mathbf Z_S=\mathbf z_S).
\end{aligned}
\]
These are precisely the endpoint-coordinate posteriors of the
restricted bridge between \(P_{0,S}\) and \(P_{1,S}\).
Taking twice their difference yields
\[
    F_i^\ast(\mathbf z,v)
    =
    F_{S,i}^\ast(\mathbf z_S,v),
    \qquad i\in S.
\]

For \(i\in S^c\), the irrelevant endpoint vectors are independent
draws from \(R\).
Their joint law and midpoint bridge are invariant under exchanging
the endpoints, so their coordinate posteriors conditional on
\(\mathbf Z_{S^c}\) agree.
Independence of the two endpoint--midpoint triples gives the same
equality conditional on the complete midpoint:
\[
    \mathbb P(X_0^i=v\mid\mathbf Z=\mathbf z)
    =
    \mathbb P(X_1^i=v\mid\mathbf Z=\mathbf z).
\]
Thus \(F_i^\ast(\mathbf z,v)=0\) for \(i\in S^c\), proving
\eqref{eq:irrelevant-flux-invariance}.

Finally, \(\mathbf Z_S\) has the midpoint law of the restricted
bridge between \(P_{0,S}\) and \(P_{1,S}\).
Substituting the field identities into the discrepancy gives
\[
\begin{aligned}
    D^\ast(P_0,P_1)
    &=
    \mathbb E
    \left[
        \sum_{i\in S}
        \sum_{v\neq Z^i}
        F_{S,i}^\ast(\mathbf Z_S,v)^2
    \right]=D^\ast(P_{0,S},P_{1,S}).
\end{aligned}
\]
The same substitution in the coordinate contributions gives
\(D_i^\ast(P_0,P_1)=0\) for every \(i\in S^c\), and preserves
the contributions on \(S\).
\end{proof}
\section{Proofs for the Theoretical Analysis}
\label{app:theoretical-analysis}

We first recall all objects used in the proofs.  Let \(\mathcal V\) be a finite
alphabet and let \(\mathbf X _0\sim P_0\) and \(\mathbf X _1\sim P_1\) be independent
random vectors in \(\mathcal V^d\).  Conditional on this single pair of full
endpoint vectors, draw independent mask bits
\(M_i\sim\operatorname{Bernoulli}(0.5)\) and set
\[
    Z^i
    =
    \begin{cases}
        X_0^i, & M_i=0,\\
        X_1^i, & M_i=1.
    \end{cases}
\]
The endpoint pair is shared by all coordinates of \(\mathbf Z\); only the mask bit
varies with \(i\).  Write
\(p_{0.5}(\mathbf z)=\mathbb P(\mathbf Z=\mathbf z)\).  For
\(p_{0.5}(\mathbf z)>0\),
\[
    q_{r,i}^*(b\mid\mathbf z)
    =
    \mathbb P(X_r^i=b\mid\mathbf Z=\mathbf z),
    \qquad r\in\{0,1\}.
\]
For \(b\neq z_i\),
\[
    F_i^*(\mathbf z,b)
    =
    2\left[
        q_{1,i}^*(b\mid\mathbf z)
        -q_{0,i}^*(b\mid\mathbf z)
    \right],
    \qquad
    \widehat F_i(\mathbf z,b)
    =
    2\left[
        \widehat q_{1,i}(b\mid\mathbf z)
        -\widehat q_{0,i}(b\mid\mathbf z)
    \right].
\]
We set \(F_i^*(\mathbf z,b)=0\) off the support of \(p_{0.5}\).  The
population discrepancy is
\[
    D ^*
    =
    \mathbb E
    \left[
        \sum_{i=1}^d
        \sum_{b\neq Z^i}
        F_i^*(\mathbf Z,b)^2
    \right].
\]

For the finite-sample construction, the endpoint datasets are
\(\{\mathbf X _{0,j}\}_{j=1}^n\) and \(\{\mathbf X _{1,k}\}_{k=1}^n\).  The two datasets
are independent, and each is i.i.d. from its corresponding endpoint law.  For
each \(\ell=1,\ldots,m\), draw one index pair
\((J_\ell,K_\ell)\) uniformly from \(\{1,\ldots,n\}^2\), with replacement,
and draw a fresh coordinate-wise mask.  The selected full endpoint vectors
are shared by all coordinates of the resulting midpoint:
\[
    Z_\ell^i
    =
    \begin{cases}
        X_{0,J_\ell}^i, & M_{\ell i}=0,\\
        X_{1,K_\ell}^i, & M_{\ell i}=1.
    \end{cases}
\]
All endpoint, index, and mask draws have the independence stated in the main
text.  Because a realized endpoint value sampled from a discrete law has
positive probability under that law, the midpoint produced by any realized
endpoint pair and mask has population probability at least
\(P_0(\mathbf x_0)P_1(\mathbf x_1)2^{-d}>0\).  Thus every sampled
\(\mathbf Z_\ell\) lies in the support of \(p_{0.5}\) almost surely.

The learned empirical discrepancy and the two conditional excess
cross-entropies are
\[
    \widehat{D }_{n,m}
    =
    \frac1m
    \sum_{\ell=1}^m
    \sum_{i=1}^d
    \sum_{b\neq Z_\ell^i}
    \widehat F_i(\mathbf Z_\ell,b)^2
\]
and
\[
    \widehat{\mathcal E}_r
    =
    \frac1{md}
    \sum_{\ell=1}^m
    \sum_{i=1}^d
    \operatorname{KL}\!\left(
        q_{r,i}^*(\cdot\mid\mathbf Z_\ell)
        \,\middle\|\,
        \widehat q_{r,i}(\cdot\mid\mathbf Z_\ell)
    \right),
    \qquad r\in\{0,1\}.
\]
The sampling radius is
\[
    R_{n,m}(\delta)
    =
    4d
    \sqrt{
        \left(
            \frac1n+\frac1{2m}
        \right)
        \log\frac{2}{\delta}
    }.
\]

\subsection{Inequalities used in the proof}

\begin{lemma}[Pinsker's inequality]
\label{lem:pinsker}
Let \(p\) and \(q\) be probability distributions on the same finite
alphabet.  With natural logarithms,
\[
    \operatorname{TV}(p,q)^2
    \leq
    \frac12\operatorname{KL}(p\|q),
    \qquad
    \operatorname{TV}(p,q):=\frac12\sum_b|p_b-q_b|.
\]
If \(q_b=0\) for some \(b\) with \(p_b>0\), the KL divergence is infinite
and the inequality remains valid in the extended-real sense.
\end{lemma}

\begin{lemma}[Categorical squared error controlled by KL]
\label{lem:categorical-l2-kl}
For probability vectors \(p=(p_b)_b\) and \(q=(q_b)_b\) on the same finite
alphabet,
\[
    \sum_b(p_b-q_b)^2
    \leq
    \operatorname{KL}(p\|q).
\]
\end{lemma}
\begin{proof}
Set \(d_b=p_b-q_b\).  Since both vectors have total mass one,
\(\sum_b d_b=0\).  Hence the total positive mass of \((d_b)_b\) equals the
total negative mass in absolute value, and both equal
\(\operatorname{TV}(p,q)\).  For nonnegative numbers, the sum of squares is
at most the square of the sum, so
\[
    \sum_b d_b^2
    \leq
    \left(\sum_{d_b>0}d_b\right)^2
    +
    \left(\sum_{d_b<0}(-d_b)\right)^2
    =
    2\operatorname{TV}(p,q)^2.
\]
The assumptions of Lemma~\ref{lem:pinsker} hold because \(p\) and \(q\) are
probability distributions on the same finite alphabet.  Applying that lemma
to the last display gives
\(\sum_b d_b^2\leq\operatorname{KL}(p\|q)\), as claimed.
\end{proof}

\begin{lemma}[Perturbation of squared Euclidean norms]
\label{lem:squared-norm-perturbation}
For any two finite Euclidean vectors \(u\) and \(v\),
\[
    \left|\|u\|_2^2-\|v\|_2^2\right|
    \leq
    2\|u\|_2\|u-v\|_2+\|u-v\|_2^2.
\]
\end{lemma}
\begin{proof}
The reverse triangle inequality gives
\(|\|u\|_2-\|v\|_2|\leq\|u-v\|_2\), while the triangle inequality gives
\(\|v\|_2\leq\|u\|_2+\|u-v\|_2\).  Therefore,
\[
    \left|\|u\|_2^2-\|v\|_2^2\right|
    =
    |\|u\|_2-\|v\|_2|(\|u\|_2+\|v\|_2)
    \leq
    \|u-v\|_2(2\|u\|_2+\|u-v\|_2).
\]
Expanding the final product proves the claim.
\end{proof}

\begin{lemma}[Exponential form of McDiarmid's inequality]
\label{lem:mcdiarmid-mgf}
Let \(G(Y_1,\ldots,Y_N)\) be a function of independent random variables.  If
replacing only \(Y_s\) can change \(G\) by at most \(c_s\), then for every
\(\lambda\in\mathbb R\),
\[
    \mathbb E\exp\!\left(\lambda[G-\mathbb EG]\right)
    \leq
    \exp\!\left(\frac{\lambda^2}{8}\sum_{s=1}^N c_s^2\right).
\]
\end{lemma}

\begin{lemma}[Hoeffding's lemma]
\label{lem:hoeffding-lemma}
If a real random variable \(Y\) has mean zero and satisfies
\(Y\in[a,b]\) almost surely, then for every \(\lambda\in\mathbb R\),
\[
    \mathbb E e^{\lambda Y}
    \leq
    \exp\!\left(\frac{\lambda^2(b-a)^2}{8}\right).
\]
\end{lemma}

\begin{proof}[Proof of Theorem~\ref{thm:finite-sample-discrepancy}]
The proof has four steps.  Step~1 derives a uniform envelope for the exact
normalized flux.  Step~2 converts the two oracle excess cross-entropies into a
deterministic plug-in bound expressed through the observed learned
discrepancy.  Step~3 controls the exact discrepancy computed from the finite
endpoint pools and random bridge draws.  Step~4 combines the deterministic and
probabilistic bounds.

All logarithms are natural.  If a learned head assigns zero probability to a
symbol having positive exact posterior probability at a sampled state, the
corresponding KL divergence is infinite and the theorem is vacuous.  We thus
consider the nontrivial case in which both headwise errors are finite.

\paragraph{Step 1: an alphabet-free envelope for the exact flux.}
Fix \(\mathbf z\) with \(p_{0.5}(\mathbf z)>0\) and a coordinate \(i\).  On
the event \(\{\mathbf Z=\mathbf z\}\), the value \(Z^i=z_i\) is copied from either
\(X_0^i\) or \(X_1^i\).  Therefore both endpoint coordinates cannot
simultaneously differ from \(z_i\), and pointwise on this event,
\[
    \mathbf 1\{X_0^i\neq z_i\}
    +
    \mathbf 1\{X_1^i\neq z_i\}
    \leq1.
\]
Taking conditional expectation given \(\mathbf Z=\mathbf z\) preserves the
inequality and gives
\begin{equation}
    \sum_{b\neq z_i}q_{0,i}^*(b\mid\mathbf z)
    +
    \sum_{b\neq z_i}q_{1,i}^*(b\mid\mathbf z)
    \leq1.
    \label{eq:appendix-combined-off-diagonal-mass}
\end{equation}
For a finite vector, its squared \(\ell_2\) norm is at most the square of its
\(\ell_1\) norm.  Moreover, the posterior probabilities are nonnegative, so
\(|q_{1,i}^*-q_{0,i}^*|\leq q_{1,i}^*+q_{0,i}^*\) componentwise.
Applying these facts and then
\eqref{eq:appendix-combined-off-diagonal-mass} yields
\begin{align}
    \sum_{b\neq z_i}F_i^*(\mathbf z,b)^2
    &\leq
    4\left(
        \sum_{b\neq z_i}
        \left|
            q_{1,i}^*(b\mid\mathbf z)
            -q_{0,i}^*(b\mid\mathbf z)
        \right|
    \right)^2
    \nonumber\\
    &\leq
    4\left(
        \sum_{b\neq z_i}q_{1,i}^*(b\mid\mathbf z)
        +
        \sum_{b\neq z_i}q_{0,i}^*(b\mid\mathbf z)
    \right)^2
    \leq4.
    \label{eq:appendix-exact-coordinate-envelope}
\end{align}
Thus, for every state,
\begin{equation}
    0
    \leq
    \sum_{i=1}^d
    \sum_{b\neq z_i}
    F_i^*(\mathbf z,b)^2
    \leq4d.
    \label{eq:appendix-exact-summary-envelope}
\end{equation}

\paragraph{Step 2: deterministic oracle plug-in stability.}
For each sampled state, coordinate, and head, the exact and learned posterior
vectors are categorical distributions on the same finite alphabet.  Therefore
the assumptions of Lemma~\ref{lem:categorical-l2-kl} are satisfied.  Applying
that lemma with
\(p=q_{r,i}^*(\cdot\mid\mathbf Z_\ell)\) and
\(q=\widehat q_{r,i}(\cdot\mid\mathbf Z_\ell)\), removing the nonnegative term
corresponding to \(b=Z_\ell^i\), and then summing over \(i\) and \(\ell\)
gives
\begin{equation}
    \frac1m
    \sum_{\ell=1}^m
    \sum_{i=1}^d
    \sum_{b\neq Z_\ell^i}
    \left[
        \widehat q_{r,i}(b\mid\mathbf Z_\ell)
        -q_{r,i}^*(b\mid\mathbf Z_\ell)
    \right]^2
    \leq
    d\widehat{\mathcal E}_r.
    \label{eq:appendix-head-error}
\end{equation}
The factor \(d\) appears because \(\widehat{\mathcal E}_r\) averages over
coordinates, whereas the left-hand side sums over them.

Form one finite vector from the entries
\(m^{-0.5}[\widehat F_i(\mathbf Z_\ell,b)-F_i^*(\mathbf Z_\ell,b)]\), indexed by all
triples \((\ell,i,b)\) with \(b\neq Z_\ell^i\).  By the definition of the
fluxes, this vector is twice the difference between the corresponding error
vectors of the two heads.  The Euclidean triangle inequality therefore gives,
after using \eqref{eq:appendix-head-error} for \(r=0\) and \(r=1\),
\begin{equation}
    \left[
        \frac1m
        \sum_{\ell=1}^m
        \sum_{i=1}^d
        \sum_{b\neq Z_\ell^i}
        \left(
            \widehat F_i(\mathbf Z_\ell,b)
            -F_i^*(\mathbf Z_\ell,b)
        \right)^2
    \right]^{0.5}
    \leq
    2\sqrt d
    \left(
        \sqrt{\widehat{\mathcal E}_0}
        +\sqrt{\widehat{\mathcal E}_1}
    \right).
    \label{eq:appendix-flux-error}
\end{equation}

Define the exact-flux average on these same sampled midpoint states by
\begin{equation}
    \widetilde{D }_{n,m}^*
    :=
    \frac1m
    \sum_{\ell=1}^m
    \sum_{i=1}^d
    \sum_{b\neq Z_\ell^i}
    F_i^*(\mathbf Z_\ell,b)^2.
    \label{eq:appendix-exact-empirical-summary}
\end{equation}
To apply Lemma~\ref{lem:squared-norm-perturbation}, take \(u\) to be the
finite vector with components \(m^{-0.5}\widehat F_i(\mathbf Z_\ell,b)\) and \(v\)
the vector with components \(m^{-0.5}F_i^*(\mathbf Z_\ell,b)\).  Then
\(\|u\|_2^2=\widehat{D }_{n,m}\),
\(\|v\|_2^2=\widetilde{D }_{n,m}^*\), and
\(\|u-v\|_2\) is the left-hand side of \eqref{eq:appendix-flux-error}.
Consequently, Lemma~\ref{lem:squared-norm-perturbation} and
\eqref{eq:appendix-flux-error} give
\begin{equation}
    \left|
        \widehat{D }_{n,m}
        -\widetilde{D }_{n,m}^*
    \right|
    \leq
    4\sqrt{d\widehat{D }_{n,m}}
    \left(
        \sqrt{\widehat{\mathcal E}_0}
        +\sqrt{\widehat{\mathcal E}_1}
    \right)
    +
    4d
    \left(
        \sqrt{\widehat{\mathcal E}_0}
        +\sqrt{\widehat{\mathcal E}_1}
    \right)^2.
    \label{eq:appendix-data-dependent-plugin-bound}
\end{equation}
This is a deterministic inequality for the realized samples and heads.  In
particular, it does not require the heads to be independent of the midpoint
samples on which the discrepancy is evaluated.

\paragraph{Step 3: concentration of the exact empirical discrepancy.}
For full endpoint vectors \(\mathbf x_0,\mathbf x_1\) and a mask \(\mathbf M\),
let \(\mathsf B(\mathbf x_0,\mathbf x_1,\mathbf M)\) denote their
coordinate-wise midpoint.  For fixed endpoint datasets, define the exact
all-pairs-and-masks average
\begin{equation}
    \overline{D }_n^*
    :=
    \frac1{n^2}
    \sum_{j=1}^n
    \sum_{k=1}^n
    \mathbb E_{\mathbf M}
    \left[
        \sum_{i=1}^d
        \sum_{b\neq
            \mathsf B_i(\mathbf X _{0,j},\mathbf X _{1,k},\mathbf M)}
        F_i^*\!\left(
            \mathsf B(\mathbf X _{0,j},\mathbf X _{1,k},\mathbf M),b
        \right)^2
    \right].
    \label{eq:appendix-product-average}
\end{equation}
For each fixed \((j,k)\), the pair
\((\mathbf X _{0,j},\mathbf X _{1,k})\) has law \(P_0\times P_1\).  Averaging over the
independent midpoint mask therefore gives the population midpoint law.
Linearity of expectation, which does not require the \(n^2\) terms to be
independent, yields
\begin{equation}
    \mathbb E\overline{D }_n^*
    =
    D ^*.
    \label{eq:appendix-product-average-mean}
\end{equation}

We now apply Lemma~\ref{lem:mcdiarmid-mgf}.  Its independent inputs are the
\(2n\) endpoint observations, not the dependent \(n^2\) cross-pair terms.  If
one observation \(\mathbf X _{0,j}\) is replaced, only the \(n\) summands with first
index \(j\) can change.  Each mask-averaged summand lies in \([0,4d]\) by
\eqref{eq:appendix-exact-summary-envelope}, so its old and new values differ
by at most \(4d\).  After division by \(n^2\), the total change is at most
\(4d/n\).  The same argument holds for each \(\mathbf X _{1,k}\).  Thus the bounded-
difference constants required by Lemma~\ref{lem:mcdiarmid-mgf} are all
\(4d/n\), and
\[
    \sum_{s=1}^{2n}c_s^2
    =
    2n\left(\frac{4d}{n}\right)^2
    =
    \frac{32d^2}{n}.
\]
Applying the lemma and using \eqref{eq:appendix-product-average-mean} gives
\begin{equation}
    \mathbb E
    \exp\!\left(
        \lambda[
            \overline{D }_n^*-D ^*
        ]
    \right)
    \leq
    \exp\!\left(\frac{4\lambda^2d^2}{n}\right),
    \qquad \lambda\in\mathbb R.
    \label{eq:appendix-endpoint-mgf}
\end{equation}

Next condition on the endpoint datasets.  Under this conditioning, the triples
\((J_\ell,K_\ell,\mathbf M_\ell)\) are independent and identically distributed
across \(\ell\).  Hence the \(m\) exact-flux summaries averaged in
\(\widetilde{D }_{n,m}^*\) are conditionally i.i.d., each lies in
\([0,4d]\) by \eqref{eq:appendix-exact-summary-envelope}, and their
conditional mean is \(\overline{D }_n^*\).  For each centered
summary, the assumptions of Lemma~\ref{lem:hoeffding-lemma} hold conditionally:
its conditional mean is zero and its range has width \(4d\).  Applying that
lemma after division by \(m\), and multiplying the resulting conditional
moment bounds by conditional independence, gives
\begin{equation}
    \mathbb E\!\left[
        \exp\!\left(
            \lambda[
                \widetilde{D }_{n,m}^*
                -\overline{D }_n^*
            ]
        \right)
        \,\middle|\,
        \{\mathbf X _{0,j}\}_{j=1}^n,
        \{\mathbf X _{1,k}\}_{k=1}^n
    \right]
    \leq
    \exp\!\left(\frac{2\lambda^2d^2}{m}\right).
    \label{eq:appendix-bridge-mgf}
\end{equation}

The two fluctuations satisfy
\[
    \widetilde{D }_{n,m}^*-D ^*
    =
    \left(
        \widetilde{D }_{n,m}^*
        -\overline{D }_n^*
    \right)
    +
    \left(
        \overline{D }_n^*-D ^*
    \right).
\]
Conditioning first on the endpoint datasets makes the second difference fixed,
so \eqref{eq:appendix-bridge-mgf} applies to the first.  Taking the outer
expectation and then using \eqref{eq:appendix-endpoint-mgf} gives
\begin{equation}
    \mathbb E
    \exp\!\left(
        \lambda[
            \widetilde{D }_{n,m}^*-D ^*
        ]
    \right)
    \leq
    \exp\!\left[
        4\lambda^2d^2
        \left(
            \frac1n+\frac1{2m}
        \right)
    \right],
    \qquad \lambda\in\mathbb R.
    \label{eq:appendix-combined-mgf}
\end{equation}

For \(t>0\) and \(\lambda>0\), apply Markov's inequality to the nonnegative
random variable
\(\exp(\lambda[\widetilde{D }_{n,m}^*-D ^*])\).
Using \eqref{eq:appendix-combined-mgf} and minimizing the resulting quadratic
exponent over \(\lambda\) gives
\[
    \mathbb P\!\left(
        \widetilde{D }_{n,m}^*-D ^*\geq t
    \right)
    \leq
    \exp\!\left[
        -\frac{t^2}{
            16d^2\left(\frac1n+\frac1{2m}\right)}
    \right].
\]
The moment bound also holds for negative \(\lambda\), so the same calculation
controls the lower tail.  Adding the two tail probabilities and choosing
\(t=R_{n,m}(\delta)\) yields, with probability at least \(1-\delta\),
\begin{equation}
    \left|
        \widetilde{D }_{n,m}^*-D ^*
    \right|
    \leq
    R_{n,m}(\delta).
    \label{eq:appendix-combined-concentration}
\end{equation}

\paragraph{Step 4: completion of the theorem.}
On the event in \eqref{eq:appendix-combined-concentration}, the ordinary
triangle inequality gives
\[
    \left|
        D ^*-\widehat{D }_{n,m}
    \right|
    \leq
    \left|
        D ^*-\widetilde{D }_{n,m}^*
    \right|
    +
    \left|
        \widetilde{D }_{n,m}^*
        -\widehat{D }_{n,m}
    \right|.
\]
The first term is bounded by \eqref{eq:appendix-combined-concentration}, and
the second by the deterministic plug-in inequality
\eqref{eq:appendix-data-dependent-plugin-bound}.  Substituting both bounds
gives \eqref{eq:finite-sample-discrepancy-bound}.

The concentration event depends only on the exact flux and the random sampling
scheme.  The plug-in inequality is deterministic once the samples and heads
are fixed.  Therefore, the same event supports the conclusion simultaneously
for every pair of categorical heads, even when those heads were trained using
the same endpoint, index, and mask samples.
\end{proof}

\begin{proof}[Proof of Corollary~\ref{cor:null-discrepancy}]
If \(P_0=P_1\), the first direction of
Proposition~\ref{prop:population-properties} shows that
\(F_i^*(\mathbf z,b)=0\) on every midpoint state in the support.  Every
sampled \(\mathbf Z_\ell\) lies in this support almost surely.  Hence
\[
    \widehat{D }_{n,m}
    =
    \frac1m
    \sum_{\ell=1}^m
    \sum_{i=1}^d
    \sum_{b\neq Z_\ell^i}
    \left[
        \widehat F_i(\mathbf Z_\ell,b)-F_i^*(\mathbf Z_\ell,b)
    \right]^2.
\]
Squaring \eqref{eq:appendix-flux-error} gives
\[
    \widehat{D }_{n,m}
    \leq
    4d
    \left(
        \sqrt{\widehat{\mathcal E}_0}
        +\sqrt{\widehat{\mathcal E}_1}
    \right)^2.
\]
No sampling term is needed because the exact normalized flux is identically
zero on every midpoint state that can be generated.
\end{proof}

\section{Experimental Implementation Details}
\label{app:experimental-implementation}

Unless specified otherwise, Zero Flux samples endpoint observations
independently with replacement from the two empirical pools. Each midpoint
coordinate copies either endpoint with probability one half. We use the
cross-entropy objective in Section~4.3, averaged over coordinates:
\begin{equation}
    \mathcal L_{\mathrm{CE}}
    =-\frac1d\widehat{\mathbb E}\!\left[
        \sum_{i=1}^d\left\{
            \log\widehat q_{0,i}(X_0^i\mid\mathbf Z)
            +\log\widehat q_{1,i}(X_1^i\mid\mathbf Z)
        \right\}
    \right].
    \label{eq:implementation-ce}
\end{equation}
Here and below, empirical expectations and probabilities are computed over
the indicated samples. Evaluation uses fresh endpoint pairings and midpoint
masks from the same pools. With \(m\) evaluation midpoints, we compute
\begin{equation}
\begin{aligned}
    \widehat F_i(\mathbf z,v)
    &=2\bigl[\widehat q_{1,i}(v\mid\mathbf z)
             -\widehat q_{0,i}(v\mid\mathbf z)\bigr],\\
    \widehat D_i
    &=\frac1m\sum_{\ell=1}^m\sum_{v\neq Z_\ell^i}
      \widehat F_i(\mathbf Z_\ell,v)^2,
    \qquad
    \widehat D_{n,m}=\sum_{i=1}^d\widehat D_i.
\end{aligned}
    \label{eq:implementation-score}
\end{equation}
Experiment-specific sample sizes, architectures, and training budgets are
given below.

\subsection{Sparse Field Recovery}
\label{app:ablation-study}
\begin{table}[t]
\centering
\caption{Sparse Field Recovery on \(G\). Mean \(\pm\) sd over 3 seeds.
Sign and cosine: higher is better. Amp: closer to 1 is better.
Budgets are not matched across \(d\): 5k steps (\(d\le 48\)), 20k (\(d=64\)),
40k (\(d=128\)), 80k (\(d=256\)). \(n=20{,}000\), \(\eta=0.8\).}
\label{tab:edge-reversal-g}
\begin{tabular}{llccc}
\toprule
Method & \(d\) & Sign \(\uparrow\) & Cosine \(\uparrow\) & Amp \(\to 1\) \\
\midrule
\multirow{6}{*}{Zero Flux}
 & 8   & \(1.000{\pm}0.000\) & \(0.995{\pm}0.001\) & \(0.991{\pm}0.016\) \\
 & 24  & \(1.000{\pm}0.000\) & \(0.993{\pm}0.004\) & \(1.003{\pm}0.043\) \\
 & 48  & \(1.000{\pm}0.000\) & \(0.987{\pm}0.010\) & \(1.005{\pm}0.036\) \\
 & 64  & \(1.000{\pm}0.000\) & \(0.983{\pm}0.005\) & \(0.999{\pm}0.041\) \\
 & 128 & \(1.000{\pm}0.000\) & \(0.977{\pm}0.006\) & \(1.004{\pm}0.032\) \\
 & 256 & \(1.000{\pm}0.000\) & \(0.819{\pm}0.170\) & \(0.976{\pm}0.016\) \\
\midrule
\multirow{6}{*}{Coord MSE}
 & 8   & \(1.000{\pm}0.000\) & \(0.942{\pm}0.077\) & \(1.095{\pm}0.292\) \\
 & 24  & \(1.000{\pm}0.000\) & \(0.863{\pm}0.099\) & \(1.131{\pm}0.271\) \\
 & 48  & \(0.967{\pm}0.058\) & \(0.620{\pm}0.392\) & \(0.462{\pm}0.365\) \\
 & 64  & \(1.000{\pm}0.000\) & \(0.677{\pm}0.214\) & \(1.835{\pm}2.019\) \\
 & 128 & \(1.000{\pm}0.000\) & \(0.890{\pm}0.078\) & \(0.555{\pm}0.141\) \\
 & 256 & \(0.594{\pm}0.038\) & \(0.151{\pm}0.104\) & \(0.296{\pm}0.299\) \\
\midrule
\multirow{6}{*}{Coord CE}
 & 8   & \(1.000{\pm}0.000\) & \(0.996{\pm}0.001\) & \(1.024{\pm}0.091\) \\
 & 24  & \(1.000{\pm}0.000\) & \(0.856{\pm}0.065\) & \(0.362{\pm}0.061\) \\
 & 48  & \(1.000{\pm}0.001\) & \(0.567{\pm}0.203\) & \(0.502{\pm}0.404\) \\
 & 64  & \(0.603{\pm}0.454\) & \(0.396{\pm}0.433\) & \(0.452{\pm}0.037\) \\
 & 128 & \(0.505{\pm}0.290\) & \(-0.106{\pm}0.426\) & \(0.124{\pm}0.042\) \\
 & 256 & \(0.563{\pm}0.059\) & \(0.029{\pm}0.024\) & \(0.328{\pm}0.394\) \\
\bottomrule
\end{tabular}
\label{tab:edge_full}
\end{table}

\paragraph{Sparse-field setting and training.}
For the endpoint family in Section~7.1, we use \(\eta=0.8\) and dimensions
8, 24, 48, 64, 128, and 256. The first three coordinates follow the joint
endpoint laws specified there. Every remaining coordinate is sampled
independently from \(\operatorname{Bernoulli}(1/2)\), independently of the
first three coordinates. Each method receives the same two endpoint pools of
20,000 observations and uses seeds 0, 1, and 2. The full metrics are reported in Table~\ref{tab:edge_full}

Zero Flux, Coord MSE, and Coord CE use a Transformer with embedding width
256, 3 layers, 4 attention heads, feed-forward width 1024, learned additive
positional embeddings, and zero dropout. Zero Flux uses two categorical
endpoint heads. Coord MSE uses a linear head for the regression objective
\begin{equation}
    \min_v\;\widehat{\mathbb E}\!
    \left\|v(M_c)-\bigl[\phi_c(\mathbf X_1)-\phi_c(\mathbf X_0)\bigr]
    \right\|_2^2,
    \label{eq:implementation-mse}
\end{equation}
where \(\phi_c\) and \(M_c\) are the coordinate-wise embedding and midpoint
defined in Section~7.1. Coord CE uses the same midpoint representation with
two categorical heads and loss~\eqref{eq:implementation-ce} evaluated at
\(M_c\).

All three methods use AdamW with learning rate \(3\times10^{-4}\), weight
decay \(10^{-4}\), batch size 1024, and gradient-norm clipping at 1.0.
The training budgets are 5,000 steps at dimensions 8, 24, and 48; 20,000 at
dimension 64; 40,000 at dimension 128; and 80,000 at dimension 256.
\begin{figure}
    \centering
    \includegraphics[width=0.99\linewidth]{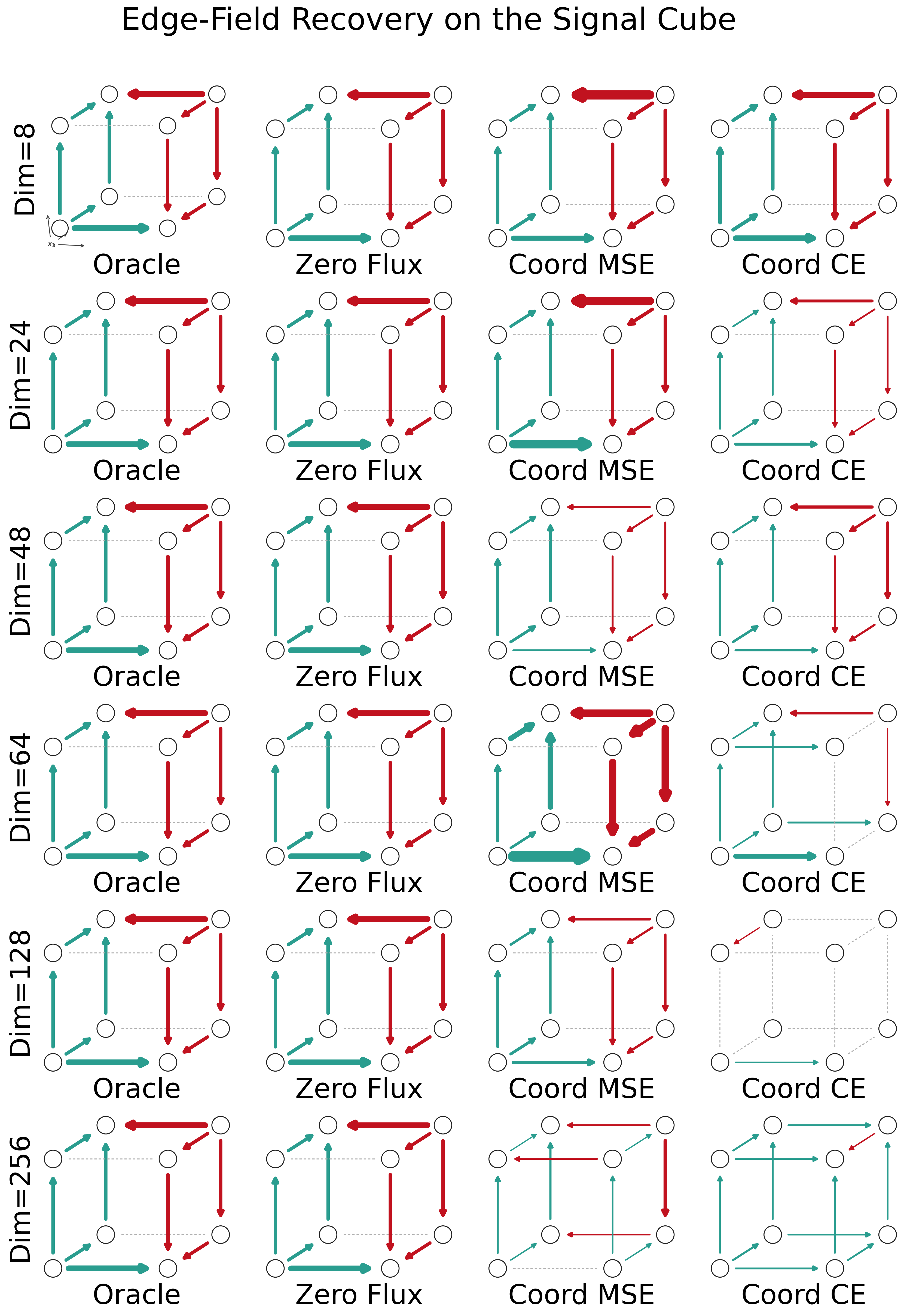}
    \caption{Sparse Field Recovery Visualization Overall Cube across Dimensions.}
    \label{fig:spar}
\end{figure}
\paragraph{Field conversion and evaluation.}
We express all predictions on the discrete log density-ratio gradient scale
used in Section~7.1. For Zero Flux, the conversion is
\[
    \widehat G_i(\mathbf z)
    = \frac{2}{\eta}\log\frac{1+\eta}{1-\eta}
      \,\widehat F_i(\mathbf z,1-z_i),
    \qquad \eta=0.8.
\]
For each coordinate-wise baseline, the input is the midpoint of a
single-coordinate flip. Let \(\widehat C_i(\mathbf z)\) be the predicted
velocity projected onto the one-hot displacement of that flip for Coord MSE,
or twice the difference between the endpoint probabilities of the flipped
value for Coord CE, with the second endpoint minus the first. The conversion is
\[
    \widehat G_i(\mathbf z)
    =2\operatorname{arctanh}\!\left(
        \operatorname{clip}\!\left(\frac{\widehat C_i(\mathbf z)}2\right)
      \right),
\]
where clipping keeps the argument inside \((-1,1)\).

All methods use the same evaluation states: the complete hypercube when it
contains at most 65,536 states, and 65,536 uniformly sampled states otherwise.
The recovery metrics are
\[
\begin{aligned}
    \mathrm{Sign}
    &=\widehat{\mathbb P}(\widehat G_iG_i>0\mid G_i\neq0),\\
    \mathrm{Cosine}
    &=\frac{\langle\widehat G,G\rangle}
            {\|\widehat G\|_2\|G\|_2},\\
    \mathrm{Amplitude\ ratio}
    &=\frac{\widehat{\mathbb E}[|\widehat G_i|\mid G_i\neq0]}
            {\widehat{\mathbb E}[|G_i|\mid G_i\neq0]}.
\end{aligned}
\]
Sign and amplitude are averaged over state-coordinate pairs with nonzero
oracle values. Cosine similarity uses all evaluated states and coordinates.
We report means and one standard
deviation across the three seeds.

For the three-coordinate cube plots, we average over all configurations of
the remaining coordinates when there are at most 4,096, and otherwise use
4,096 shared uniform draws. Each displayed edge value is half the difference
between the two directed predictions, averaged over these configurations.
Arrow widths share the oracle magnitude scale
\(2\log((1+\eta)/(1-\eta))\); values below 5\% of this scale are shown as
thin lines.
\paragraph{Field visualization.}
We summarize the recovered field on the first three coordinates as a cube.
Write \(\mathbf z=(\mathbf u,\mathbf w)\), where
\(\mathbf u\in\{0,1\}^3\) contains the signal coordinates and
\(\mathbf w\in\{0,1\}^{d-3}\) contains the independent uniform
coordinates. The eight vertices represent the signal states after averaging
over \(\mathbf w\), and the twelve edges connect states differing in one
signal coordinate. For each training seed, we compute
\[
    \widetilde G_i(\mathbf u)
    =\frac12\widehat{\mathbb E}_{\mathbf w}\!\left[
        \widehat G_i(\mathbf u,\mathbf w)
        -\widehat G_i(\mathbf u^{\oplus i},\mathbf w)
    \right],
    \qquad i\in\{1,2,3\},\quad u_i=0,
\]
where \(\mathbf u^{\oplus i}\) flips coordinate \(i\).
The expectation averages uniformly over all configurations of
\(\mathbf w\) when \(d\leq8\), and over 4,096 sampled configurations
otherwise. These samples are shared across
methods and training seeds. We then average the edge values over training
seeds 0, 1, and 2 before assigning directions, widths, and colors.
The Oracle field is antisymmetric and independent of \(\mathbf w\), so
these operations preserve its values at every dimension.

Each vertex is drawn as an unlabeled white circle at planar coordinates
\[
    (1.70u_1+0.78u_3,\;1.62u_2+0.50u_3).
\]
Each node represents a fixed configuration
\(\mathbf u=(x_1,x_2,x_3)\in\{0,1\}^3\), corresponding to all
full states \((\mathbf u,\mathbf w)\) with these three coordinates fixed.
The displayed edge values are averaged over the remaining coordinates
\(\mathbf w\). The front lower-left node corresponds to \((0,0,0)\).
Moving right, upward, or along the receding upper-right diagonal changes
\(x_1\), \(x_2\), or \(x_3\) from 0 to 1, respectively.
Columns show Oracle, Zero Flux, Coord MSE, and Coord CE, from left to right;
rows in the combined plot correspond to dimensions and are labelled
\(\mathrm{Dim}=d\). Coordinate axes appear only in the Oracle panels.

A positive mean edge value gives a teal arrow from \(u_i=0\) to
\(u_i=1\); a negative value gives a red arrow in the reverse direction.
Line width increases linearly with the absolute mean edge value, using one
common scale across all methods and dimensions. The reference magnitude
is the largest Oracle edge,
\[
    v_{\max}=2\log\frac{1+\eta}{1-\eta}
    =2\log9\approx4.39,
    \qquad \eta=0.8.
\]
Values exceeding \(v_{\max}\) produce proportionally wider arrows.
Edges with absolute mean values below \(0.05v_{\max}\) are drawn as gray
dashed lines without arrowheads. The full visualisation is provided in Figure~\ref{fig:spar}.

\subsection{Contamination Experiments}
\label{app:contamination-implementation}
\label{app:evaluation_metrics}
\label{app:structured-contamination-response}
\label{app:baseline-implementations}

\paragraph{Data and repeated runs.}
We compare the reference distribution \(P\) with
\[
    Q_\varepsilon=(1-\varepsilon)P+\varepsilon R,
    \qquad \varepsilon\in\{0,0.02,0.05,0.1,0.2\},
\]
where \(R\) is the contaminating distribution; the endpoint laws are
\(P_0=P\) and \(P_1=Q_\varepsilon\). Contamination is applied to
complete observations. Each level uses two independently sampled
pools of 1,000 observations and seeds 0 through 49. All methods receive the
same pools for a given setting and seed. Each seed includes a new draw of
both pools and a separate fit for each learned estimator.

The stationary binary Markov reference and contaminating distributions have
persistence probabilities 0.9 and 0.1, respectively, with an initial
\(\operatorname{Bernoulli}(1/2)\) state. We use dimensions 8, 64, and 128.
For the HMM experiments, the hidden chain also starts uniformly. Its
persistence probability and the probability that an emission agrees with
the hidden state are 0.95 and 0.9 for the reference distribution, and 0.788
and 1 for the contaminating distribution. Emissions are conditionally
independent given the hidden chain. Both observable distributions have
lag-one correlation 0.576 and persistence probability 0.788. We use
dimensions 64 and 128. Binarized MNIST uses flattened \(28\times28\) images,
with digit 6 as the reference and digit 9 as the contaminating distribution.

\paragraph{Evaluation protocol.}
Let \(D_s(\varepsilon)\) be a method's scalar discrepancy for seed \(s\)
at contamination level \(\varepsilon\). Each positive level is compared
with the same 50 null scores at \(\varepsilon=0\). We compute
\[
    \mathrm{AUROC}(\varepsilon)
    =\frac1{50^2}\sum_{a,b=0}^{49}
    \left[
        \mathbf1\{D_a(\varepsilon)>D_b(0)\}
        +\frac12\mathbf1\{D_a(\varepsilon)=D_b(0)\}
    \right].
\]
At threshold \(\tau\), the empirical rates are
\[
    \mathrm{FPR}(\tau)
    =\frac1{50}\sum_{s=0}^{49}\mathbf1\{D_s(0)\geq\tau\},
    \qquad
    \mathrm{TPR}_{\varepsilon}(\tau)
    =\frac1{50}\sum_{s=0}^{49}\mathbf1\{D_s(\varepsilon)\geq\tau\}.
\]
FPR@95\% TPR evaluates FPR at the highest observed threshold satisfying
\(\mathrm{TPR}_{\varepsilon}(\tau)\geq0.95\). We average each metric
equally over the four positive levels.
Response curves are divided by their seed-mean value at contamination 0.2;
the mean and one-standard-deviation band use the same scale factor. The full table is provided in Table~\ref{tab:cont_full}.

\begin{figure}
    \centering
    \includegraphics[width=0.99\linewidth]{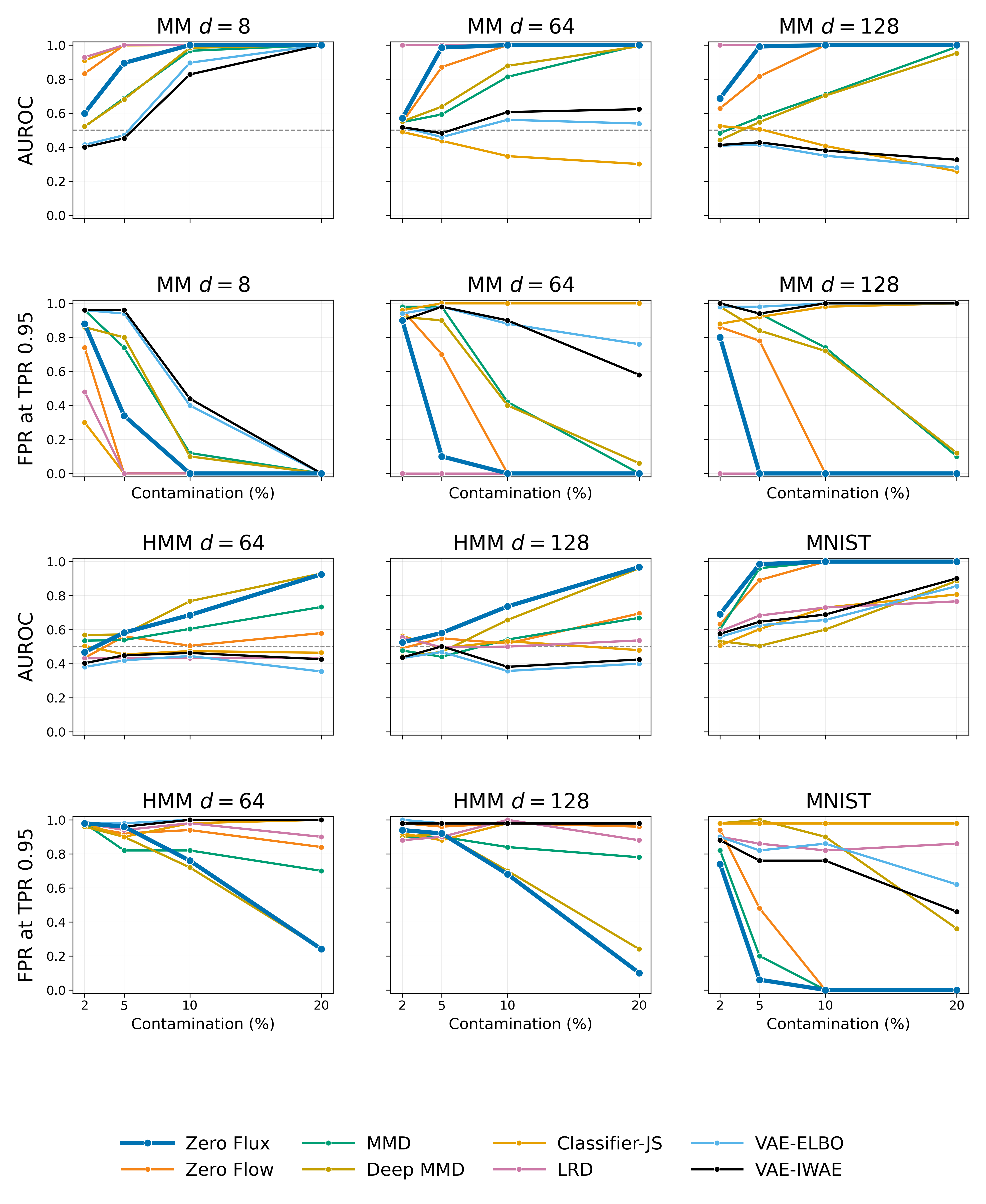}
    \caption{Contamination detection on binary Markov chains (MM), hidden
Markov models (HMM), and binarized MNIST.
The first and third rows show AUROC; the second and fourth show
FPR@95\% TPR. Metrics compare 50 runs at each positive contamination
level with 50 uncontaminated runs.
Higher AUROC and lower FPR indicate better detection.
Dashed lines indicate an AUROC of 0.5, corresponding to random guessing.}
    \label{fig:2exp_metrics}
\end{figure}
\begin{table*}[t]
\centering
\caption{
Mean detection metrics over
$\varepsilon\in\{0.02,0.05,0.1,0.2\}$,
with $50$ seeds and $n=1000$.
Each score is compared with the $50$ null runs at $\varepsilon=0$
in the same setting.
AUROC is higher-better; FPR at TPR $0.95$ is lower-better.
The best entry in each row is in bold, and the second-best is underlined.
LRD on the Markov models and HMMs is the closed-form stay-Markov
likelihood; LRD on MNIST is the causal Transformer.
}
\label{tab:s50-detection}

\centering
\resizebox{0.99\textwidth}{!}{\begin{tabular}{lcccccccc}
\hline
& \multicolumn{8}{c}{Mean AUROC $\uparrow$} \\
Setting
& Zero Flux & Zero Flow & MMD & Deep MMD
& Clf-JS & LRD & ELBO & IWAE \\
\hline
Markov $d=8$
& 0.874 & 0.958 & 0.794 & 0.797
& \underline{0.977} & \textbf{0.982} & 0.695 & 0.670 \\

Markov $d=64$
& \underline{0.889} & 0.854 & 0.738 & 0.765
& 0.394 & \textbf{1.000} & 0.519 & 0.557 \\

Markov $d=128$
& \underline{0.920} & 0.861 & 0.690 & 0.661
& 0.424 & \textbf{1.000} & 0.363 & 0.387 \\

HMM $d=64$
& \underline{0.665} & 0.520 & 0.603 & \textbf{0.709}
& 0.474 & 0.433 & 0.399 & 0.435 \\

HMM $d=128$
& \textbf{0.702} & 0.563 & 0.532 & \underline{0.632}
& 0.516 & 0.522 & 0.415 & 0.436 \\

MNIST
& \textbf{0.919} & 0.881 & \underline{0.891} & 0.631
& 0.661 & 0.692 & 0.674 & 0.703 \\
\hline

& \multicolumn{8}{c}{Mean FPR @ TPR $0.95$ $\downarrow$} \\
Setting
& Zero Flux & Zero Flow & MMD & Deep MMD
& Clf-JS & LRD & ELBO & IWAE \\
\hline
Markov $d=8$
& 0.305 & 0.185 & 0.455 & 0.440
& \textbf{0.075} & \underline{0.120} & 0.575 & 0.590 \\

Markov $d=64$
& \underline{0.250} & 0.415 & 0.595 & 0.570
& 0.990 & \textbf{0.000} & 0.890 & 0.840 \\

Markov $d=128$
& \underline{0.200} & 0.410 & 0.695 & 0.665
& 0.945 & \textbf{0.000} & 0.990 & 0.985 \\

HMM $d=64$
& \underline{0.735} & 0.915 & 0.830 & \textbf{0.705}
& 0.965 & 0.950 & 0.990 & 0.985 \\

HMM $d=128$
& \textbf{0.660} & 0.970 & 0.855 & \underline{0.690}
& 0.940 & 0.915 & 0.985 & 0.980 \\

MNIST
& \textbf{0.200} & 0.355 & \underline{0.255} & 0.810
& 0.980 & 0.860 & 0.800 & 0.715 \\
\hline
\end{tabular}}
\label{tab:cont_full}
\end{table*}
\paragraph{Zero Flux.}
We use AdamW with learning rate \(10^{-4}\), weight decay \(10^{-4}\),
gradient-norm clipping at 1.0, and zero dropout. Architecture and training
budgets depend on the input dimension:
\begin{itemize}
    \item Markov, dimension 8: a Transformer with embedding width 64,
    2 layers, 4 attention heads, feed-forward width 128, and two categorical
    heads with hidden width 64 and output width 2. Training uses 1,000 epochs
    and batch size 256; evaluation uses 30,000 midpoints.
    \item Markov and HMM, dimensions 64 and 128: a Transformer with
    embedding width 128, 4 layers, 8 attention heads, feed-forward width 512,
    and two categorical heads with hidden width 256 and output width 2.
    Training uses 1,000 epochs and batch size 256; evaluation uses 30,000
    midpoints.
    \item MNIST: a six-layer one-dimensional convolutional backbone with
    width 512, kernel width 15, and categorical-head hidden width 512.
    Training uses 120 epochs and batch size 128; evaluation uses 10,000
    midpoints.
\end{itemize}

\paragraph{Coordinate-wise Zero Flow.}
The backbone, optimizer, batch size, training budget, and number of evaluation
pairs match Zero Flux for each setting. Inputs encode each coordinate-wise
midpoint as 0, 1, or a mixed symbol. A velocity head with the same hidden
width as a Zero Flux endpoint head uses the regression objective
in~\eqref{eq:implementation-mse}. Predictions are constrained to zero on coordinates where the
endpoints agree. The score is
\(\widehat{\mathbb E}\|\widehat v(M_c)\|_2^2\).

\paragraph{MMD.}
We use an equal-weight average of exponential Hamming kernels,
\[
    k_\tau(\mathbf x,\mathbf y)
    =\exp\!\left[-H(\mathbf x,\mathbf y)/\tau\right],
    \qquad
    H(\mathbf x,\mathbf y)
    =\frac1d\sum_{i=1}^d\mathbf 1\{x_i\neq y_i\},
\]
with bandwidths 0.125, 0.25, 0.5, 1, and 2. For their average kernel \(k\)
and two endpoint samples of size \(n\), the biased estimator is
\[
    \widehat{\mathrm{MMD}}_b^2
    =\frac1{n^2}\sum_{a,b=1}^n\left[
        k(\mathbf X_{0,a},\mathbf X_{0,b})
        +k(\mathbf X_{1,a},\mathbf X_{1,b})
        -2k(\mathbf X_{0,a},\mathbf X_{1,b})
    \right].
\]
The score is \(\operatorname{clip}(\widehat{\mathrm{MMD}}_b^2/2,0,1)\).

\paragraph{Deep MMD.}
Each endpoint pool is split equally into kernel-training and held-out
evaluation samples. The feature encoder \(f\) is a one-dimensional CNN on
Markov and HMM data and a two-dimensional CNN on MNIST, with hidden width 64
and output width 32. Using the fixed Hamming kernel \(k\) above, the learned
kernel is
\[
    k_{\mathrm{deep}}(\mathbf x,\mathbf y)
    =\left[
        (1-\alpha)\exp\!\left(
            -\frac{\|f(\mathbf x)-f(\mathbf y)\|_2^2}{32\tau_f}
        \right)+\alpha
    \right]k(\mathbf x,\mathbf y),
\]
where \(\alpha\) and \(\tau_f\) are the learned mixture weight and
bandwidth. The division by 32 gives the mean squared feature distance.

Training maximizes
\[
    \widehat J
    =\frac{\widehat{\mathrm{MMD}}_u^2}
           {\sqrt{\widehat{\mathrm{var}}_{H_1}+\lambda}},
\]
where the denominator uses the estimated variance under the alternative
with regularizer \(\lambda\). The unbiased squared MMD omits within-sample
diagonal terms and uses denominator \(n(n-1)\) for those terms, while the
cross-sample denominator remains \(n^2\). We use 2,000 Adam
steps, learning rate \(10^{-3}\), and batch size 256 at dimension 8 and 512
otherwise. Evaluation uses the same clipped score as MMD on the held-out
samples.

\paragraph{Classifier-JS.}
The classifier uses the corresponding Zero Flux backbone. Writing \(h\)
for its predicted probability of the second endpoint, the balanced loss and
normalized score are
\[
    \mathcal L_{\mathrm{BCE}}
    =-\frac12\widehat{\mathbb E}_{P}\log(1-h(\mathbf X))
     -\frac12\widehat{\mathbb E}_{Q_\varepsilon}\log h(\mathbf X),
    \qquad
    1-\frac{\mathcal L_{\mathrm{BCE}}}{\log2}.
\]
Optimization uses AdamW with learning
rate \(10^{-4}\), weight decay \(10^{-4}\), and gradient-norm clipping at
1.0. Dropout is 0.1 for dimension-8 Markov data and zero otherwise. MNIST
uses batch size 128.

The maximum training budgets are 1,500 epochs for dimension-8 Markov data,
1,000 epochs for dimensions 64 and 128, and 500 epochs for MNIST. Early
stopping monitors cross-entropy on the training pools every five epochs,
with a required improvement of \(10^{-4}\). It is enabled after 200 epochs
for Markov and HMM data and after 100 epochs for MNIST. Patience is 150
evaluations for Markov and HMM data and 100 for MNIST.

\paragraph{VAE-ELBO and VAE-IWAE.}
For each contamination level and seed, we fit one shared VAE on the pooled
samples and one private VAE per endpoint. Private models are initialized
from the selected shared checkpoint. Each VAE has a diagonal Gaussian
encoder, a standard normal prior, and a factorized Bernoulli decoder. The
MLP architecture uses three SiLU layers. Markov and HMM models use latent
width 16, hidden width 256, and batch size 256; MNIST uses latent width 64, hidden width 512,
and batch size 128.

Writing \(\mathbf h\) for the latent variable, \(q(\mathbf h\mid\mathbf x)\)
for the encoder, and \(p(\mathbf x\mid\mathbf h)\) for the decoder,
each model maximizes the average training ELBO,
\[
    \mathrm{ELBO}(\mathbf x)
    =\mathbb E_{q(\mathbf h\mid\mathbf x)}
        [\log p(\mathbf x\mid\mathbf h)]
     -\mathrm{KL}\!\left(q(\mathbf h\mid\mathbf x)\Vert p(\mathbf h)\right),
    \qquad p(\mathbf h)=\mathcal N(\mathbf0,I).
\]
Training uses Adam with learning rate \(10^{-4}\), three restarts, and
ELBO-based early stopping with a minimum of 300 epochs, patience 150, and a
maximum of 500 epochs. Writing \(\mathrm{ELBO}_r\) for the private model
fitted to \(P_r\), the score in nats is
\[
    \sum_{r=0}^1\widehat{\mathbb E}_{P_r}
    \left[\mathrm{ELBO}_r(\mathbf X)
          -\mathrm{ELBO}_{\mathrm{shared}}(\mathbf X)\right].
\]
VAE-IWAE evaluates the same checkpoints and replaces each ELBO in the score
with
\[
    \mathrm{IWAE}_4(\mathbf x)
    =\mathbb E\!\left[
        \log\!\left(
            \frac14\sum_{k=1}^4
            \frac{p(\mathbf x\mid\mathbf h_k)p(\mathbf h_k)}
                 {q(\mathbf h_k\mid\mathbf x)}
        \right)
    \right],
\]
where the four latent samples are drawn independently from the encoder.

\paragraph{LRD.}
We fit a shared likelihood model \(\widehat p_{\mathrm{shared}}\) on the
pooled samples and a private model \(\widehat p_r\) for each endpoint.
Each fit minimizes the average negative log likelihood
\(-\widehat{\mathbb E}\log p(\mathbf X)\) on its training samples.
The discrepancy is
\[
    \sum_{r=0}^1\widehat{\mathbb E}_{P_r}
    \log\frac{\widehat p_r(\mathbf X)}
              {\widehat p_{\mathrm{shared}}(\mathbf X)}.
\]
Markov and HMM data use a first-order stationary binary Markov family with a uniform
initial state and persistence fitted by the empirical stay frequency.
MNIST uses causal Transformers with width 128, 4 layers, 8 attention heads,
and feed-forward width 512. They are trained with Adam, learning rate
\(10^{-4}\), batch size 128, and 500 epochs.

\subsection{UCI Categorical Joint-Dependence Benchmark}
\label{app:uci-categorical-dependence}

\paragraph{Datasets and preprocessing.}
We use the eleven categorical datasets below, with seeds 0 through 4.
Constant columns and unique identifiers are removed. Missing values, empty
strings, and question marks share one categorical value, and integer-coded
features are treated as unordered categories. Discrepancy methods use the
feature tables; class labels are reserved for the downstream evaluation.

\begin{center}
\begin{tabular}{lrr}
\hline
Dataset & Observations & Features \\
\hline
Connect-4 & 67,557 & 42 \\
Mushroom & 8,124 & 21 \\
Chess KRKPA7 & 3,196 & 35 \\
Tic-Tac-Toe Endgame & 958 & 9 \\
National Poll on Healthy Aging & 714 & 13 \\
Congressional Voting & 435 & 16 \\
Soybean (Large) & 307 & 35 \\
Primary Tumor & 339 & 17 \\
SPECT Heart & 267 & 22 \\
Breast Cancer & 286 & 9 \\
Audiology & 226 & 69 \\
\hline
\end{tabular}
\end{center}

The removed constant features are \texttt{veil-type} in Mushroom and
\texttt{Age} in Healthy Aging. Audiology combines its 200 and 26 rows
and excludes the identifier column.
\label{sec:uci_add}
\begin{figure}
    \centering
    \includegraphics[width=\linewidth]{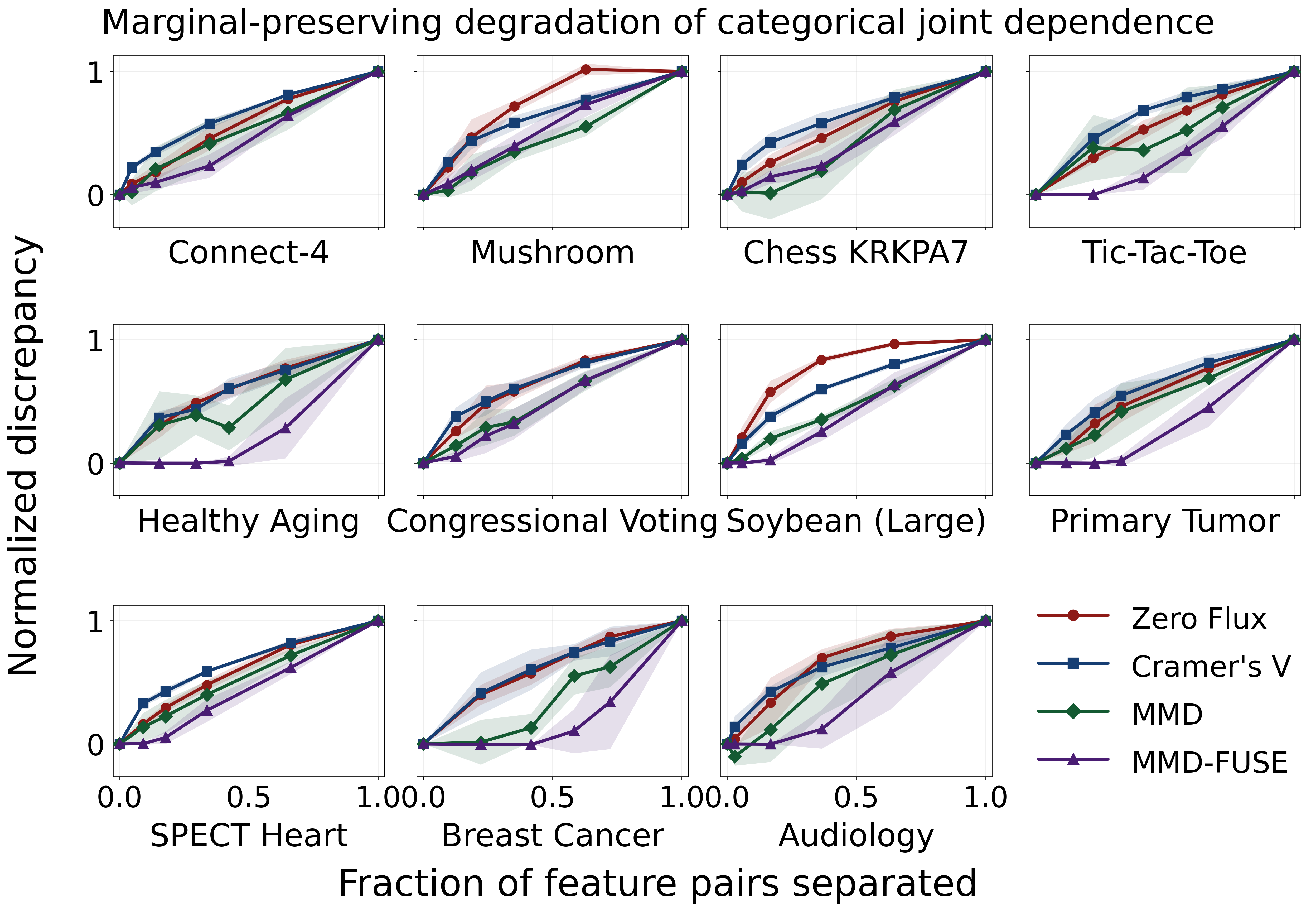}
    \caption{UCI Dependency Trajectories}
    \label{fig:uci_trajectories}
\end{figure}
\paragraph{Dependence degradation.}
Let \(P\) denote the empirical distribution of the original table and
\(Q_\ell\) the empirical distribution at degradation level \(\ell\).
A random feature order is fixed for each run. At each successive step, one
feature is separated from the remaining block. Its rows and those of the
remaining block are independently permuted, while previously separated
features retain their current row order. This preserves every univariate
empirical marginal and the joint empirical distribution within each current
block. The final table contains independently rematched columns.

Severity is the fraction of unordered feature pairs assigned to different
blocks. After \(k\) features have been separated from the remaining block,
it is
\[
    s(k)=1-\frac{(d-k)(d-k-1)}{d(d-1)}.
\]
We target six levels: 0, 0.04, 0.16, 0.36, 0.64, and 1.
Intermediate levels use successively increasing numbers of separated
features closest to these targets, with ties resolved in favor of fewer
separated features. The final level separates every feature; duplicate
levels are removed. All methods use the same degraded tables, and
trajectories are indexed by the realized severity.

\paragraph{Zero Flux.}
Each dataset, seed, and degradation level uses a separately fitted model.
The backbone has the architecture specified for Nationscape, with a single
constant conditioning value. Each of the two categorical heads has hidden
width 256 and an output width equal to the largest number of categories
among the features.
Predicted distributions are restricted to valid categories for each feature.
Zero initialization is disabled for the output heads and final adaptive
layer normalization.

We use AdamW with weight decay \(10^{-4}\), gradient-norm clipping at 1.0,
and bfloat16 mixed precision. Datasets with fewer than 8,000 observations
use 25,000 steps, batch size 512, and learning rate \(10^{-4}\). Mushroom
and Connect-4 use 50,000 steps, batch size 1,024, and learning rate
\(5\times10^{-5}\). These budgets are fixed across degradation levels.
Evaluation uses \(m=20{,}000\) fresh midpoint draws from the fitted endpoint
pools and reports \(\widehat D_{n,m}/d\) from~\eqref{eq:implementation-score}.

\paragraph{Cram\'er's V.}
For each table, we compute pairwise Cram\'er's V with Bergsma's finite-sample
bias correction. Writing \(V_{ij}\) for the corrected association between
features \(i\) and \(j\), the discrepancy is
\[
    \left[
        \frac{2}{d(d-1)}\sum_{i<j}
        \bigl(V_{ij}(P)-V_{ij}(Q_\ell)\bigr)^2
    \right]^{1/2}.
\]

\paragraph{Hamming MMD.}
We use the exponential Hamming kernel \(k_\tau\) defined above. Its bandwidth
is the median positive distance among 50,000 sampled pairs of distinct
reference rows, with a lower bound of the reciprocal feature count. The
bandwidth is fixed along the degradation path. With \(\mathbf X,\mathbf X'\)
drawn from distinct reference rows and \(\mathbf Y,\mathbf Y'\) from distinct
degraded rows, we estimate
\[
    \widehat{\mathrm{MMD}}_u^2
    =\widehat{\mathbb E}\!\left[
        k_\tau(\mathbf X,\mathbf X')+k_\tau(\mathbf Y,\mathbf Y')
        -k_\tau(\mathbf X,\mathbf Y')-k_\tau(\mathbf X',\mathbf Y)
    \right].
\]
The average uses 100,000 independently sampled quadruples. The same index
quadruples are reused across levels within a seed, and the signed estimate
is retained.

\paragraph{MMD-FUSE.}
We use the FUSE-N statistic with 20 exponential Hamming kernels. For each
seed, we select \(n'=\min(512,n)\) rows without replacement from each table
and reuse the same row positions across degradation levels. Bandwidths are linearly
spaced from half the fifth percentile to twice the ninety-fifth percentile
of Hamming distances between reference rows, with zero distances replaced
by the reference-distance median. The bandwidth grid remains fixed along
the path. Let \(K_j\) be the kernel matrix on the pooled samples for kernel
\(j\), and let \(\widehat{\mathrm{MMD}}_{u,j}^2\) be its unbiased squared
MMD. The normalized component and the final score are
\[
    u_j
    =\frac{[n'(n'-1)]^{3/2}\widehat{\mathrm{MMD}}_{u,j}^2}
           {\left[\sum_{a\neq b}(K_j)_{ab}^2\right]^{1/2}},
    \qquad
    \log\!\left(\frac1{20}\sum_{j=1}^{20}\exp(u_j)\right).
\]
The denominator sums squared kernel entries for distinct pooled
observations. Aggregation uses equal kernel weights and parameter 1.

\paragraph{Evaluation metrics.}
For each dataset and seed, let \(D_0,\ldots,D_L\) denote a method's scores
and \(s_0,\ldots,s_L\) the corresponding realized severities. We compute
\[
\begin{aligned}
    \mathrm{C\text{-}index}
    &=\frac{2}{L(L+1)}\sum_{0\leq a<b\leq L}
      \left[\mathbf1\{D_b>D_a\}+\frac12\mathbf1\{D_b=D_a\}\right],\\
    r(s,D)
    &=\frac{\sum_{\ell=0}^L(s_\ell-\bar s)(D_\ell-\bar D)}
            {\sqrt{\sum_{\ell=0}^L(s_\ell-\bar s)^2}
             \sqrt{\sum_{\ell=0}^L(D_\ell-\bar D)^2}},\\
    \mathrm{NVI}
    &=\frac{\sum_{\ell=0}^{L-1}\max(D_\ell-D_{\ell+1},0)}
           {|D_L-D_0|}.
\end{aligned}
\]
Here \(\bar s\) and \(\bar D\) are the averages over degradation levels.
These metrics measure ordering, linear association with severity, and
successive downward violations, respectively. All metrics use the raw
scores. They are averaged over five seeds within each dataset; the summary
reports the mean and standard deviation across the eleven dataset-level
averages. Trajectory plots use
\((D_\ell-D_0)/(D_L-D_0)\).

\subsection{Nationscape Partisan Discrepancy Trajectory}
\label{app:nationscape}

\paragraph{Survey data and preprocessing.}
We use 77 regular weekly waves of the Democracy Fund--UCLA Nationscape
survey, spanning 18 July 2019 to 12 January 2021. The probe file
\texttt{\_probe/ns20191226.dta} is excluded. Each entire wave is assigned
to the calendar month of its field-start date, giving 19 monthly pools from
July 2019 through January 2021. January 2021 contains one wave. Democrats
and Republicans are identified using \texttt{pid3}.

The fixed \texttt{strict14} response vector contains 14 policy items:
\texttt{wall}, \texttt{guns\_\allowbreak bg}, \texttt{mctaxes},
\texttt{estate\_\allowbreak tax}, \texttt{raise\_\allowbreak upper\_\allowbreak tax}, \texttt{college},
\texttt{israel}, \texttt{gov\_\allowbreak insurance}, \texttt{path\_\allowbreak to\_\allowbreak citizenship},
\texttt{dreamers}, \texttt{deportation}, \texttt{ban\_\allowbreak guns},
\texttt{ban\_\allowbreak assault\_\allowbreak rifles}, and \texttt{limit\_\allowbreak magazines}.
We encode opposition, support, and uncertainty as 0, 1, and 2. For items
already collapsed in the survey, raw codes 1, 2, and 999 indicate support,
opposition, and uncertainty, respectively. For five-category items, codes
1 and 2 indicate support, 3 and 4 indicate opposition, and 5 and 999
indicate uncertainty. Code 888, Stata missing values, and all other codes
are treated as missing. Respondents with any missing item are removed.

We retain positive finite values of the official \texttt{weight} variable.
Within each month and party, respondents from the assigned waves are pooled
and sampled with replacement with probability \(w_j/\sum_\ell w_\ell\),
where \(w_j\) is the original survey weight. Each monthly party sample must
contain at least 100 respondents and have Kish effective sample size
\[
    \mathrm{ESS}=\frac{(\sum_j w_j)^2}{\sum_j w_j^2}\geq100.
\]

Across months, the median numbers of retained Democrats and Republicans
are 2,774 and 2,416, with ranges 542--3,413 and 388--3,135. Their median
effective sample sizes are approximately 1,036 and 941.

\paragraph{Zero Flux.}
A single month-conditioned DDiT Transformer is shared across all 19 months. It
uses embedding width 128, 4 layers, 8 attention heads, feed-forward width
512, rotary positional embeddings, and a learned 128-dimensional month
embedding. Two zero-initialized categorical heads predict the coordinate-wise
endpoint posteriors for Democrats and Republicans. Dropout is zero.

Training uses AdamW with constant learning rate \(10^{-4}\), weight decay
\(10^{-4}\), batch size 512, gradient-norm clipping at 1.0, and bfloat16
mixed precision on GPUs. We train for 500 epochs with 89 updates per epoch.
Each update samples one month and draws the two endpoints independently
from its weighted party pools. Each midpoint coordinate copies either
endpoint with probability one half. The two heads minimize the summed
cross-entropy in~\eqref{eq:implementation-ce}, conditioned on the month.
For each month and seed, we evaluate 20,000 fresh midpoint draws from the
same weighted respondent pools. Applying~\eqref{eq:implementation-score}
within each month, the per-coordinate trajectory is
\[
    \frac{\widehat D(t)}{d}
    =\frac1d\sum_{i=1}^d\widehat D_i(t),
    \qquad d=14.
\]

\paragraph{Mixture MMD and median MMD.}
Both variants use the normalized Hamming distance \(H\) and exponential
kernel \(k_\tau\) defined above, with \(d=14\). Mixture MMD averages five
kernels with bandwidths 0.125, 0.25, 0.5, 1, and 2, fixed across months.
Median MMD selects one bandwidth separately for each month. It draws
50,000 independent pairs from the pooled two-party sample, with respondent
probabilities proportional to their original survey weights. The bandwidth
is the median positive Hamming distance among these pairs, with a lower
bound of \(1/14\). Bandwidth selection and score estimation use independent
draws.

For either kernel \(k\), the monthly score is the unbiased linear-time
estimate
\[
    \widehat{\mathrm{MMD}}^2(t)
    =\widehat{\mathbb E}\!\left[
        k(\mathbf X,\mathbf X')+k(\mathbf Y,\mathbf Y')
        -k(\mathbf X,\mathbf Y')-k(\mathbf X',\mathbf Y)
    \right].
\]
Here \(\mathbf X,\mathbf X'\) are independent weighted draws from the
Democratic pool, and \(\mathbf Y,\mathbf Y'\) are independent weighted
draws from the Republican pool. The average uses 200,000 independently
sampled quadruplets per month, evaluated in batches of 4,096. We report
the raw estimate.

\paragraph{Repeated runs and visualization.}
All methods use seeds 0, 1, and 2. Curves show the seed mean with a band of
one sample standard deviation across the three seeds. The policy-item
heatmap shows \(\widehat D_i(t)\) averaged over these seeds. The shape
comparison includes Zero Flux, mixture MMD, and median MMD. Each method's
trajectory is standardized separately over the 19 months by subtracting
its temporal mean and dividing by its temporal standard deviation.
Event markers are placed at 11 March 2020, 25 May 2020, 3 November 2020,
and 6 January 2021, corresponding to the pandemic declaration, the murder
of George Floyd, Election Day, and the Capitol attack.

\section{Additional Experimental Results}
\label{app:additional}

\subsection{Inductive Geometry}

\begin{figure}
    \centering
    \includegraphics[width=0.5\linewidth]{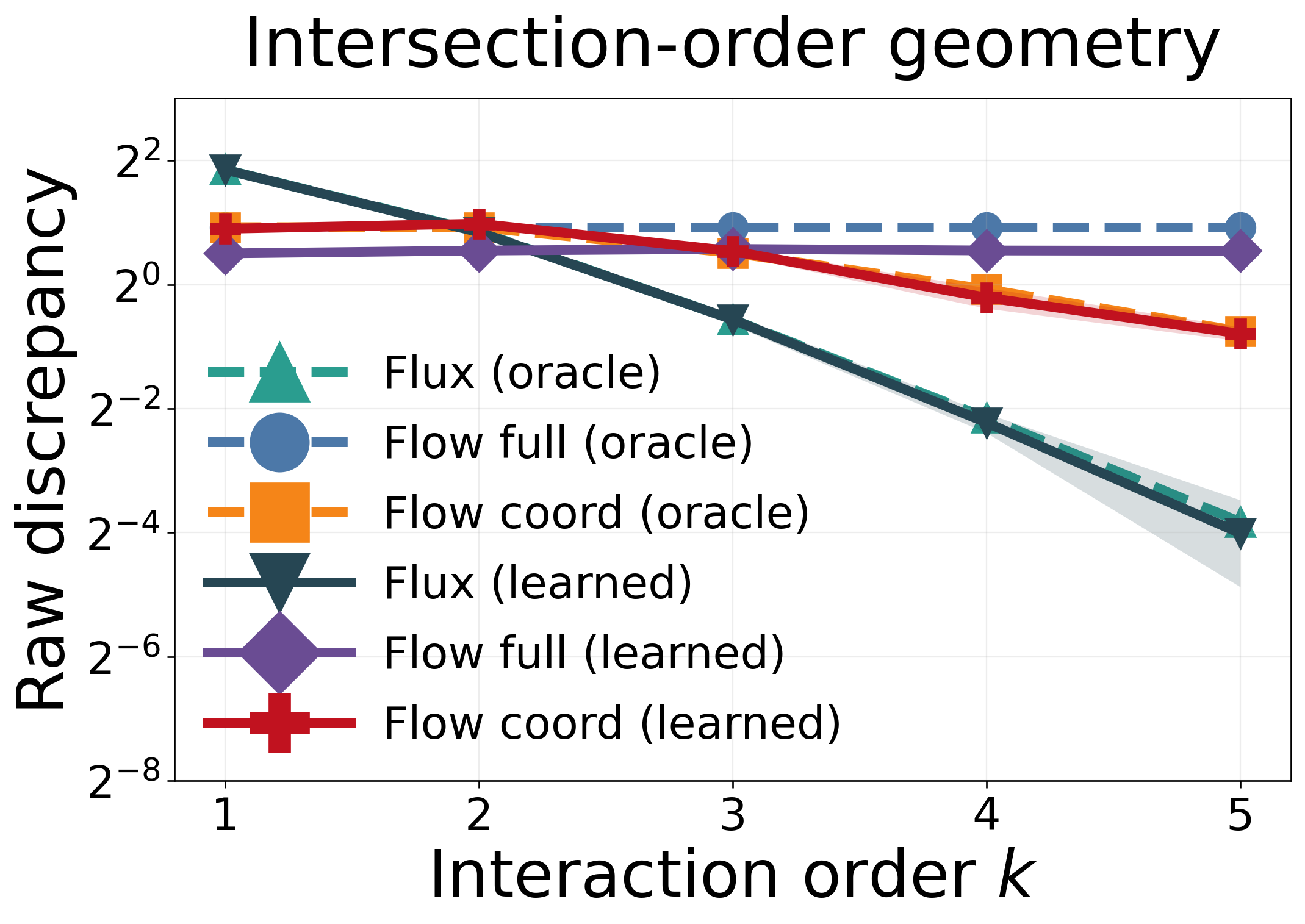}
    \caption{Sensitivity to interaction order on the binary parity family.
    Curves show population discrepancies and learned estimates; shaded bands
    indicate one standard deviation across five seeds.}
    \label{fig:interaction_order}
\end{figure}

We use \emph{inductive geometry} to describe a discrepancy's relative
sensitivity to different dependence structures. To examine how this
sensitivity varies with interaction order, we compare Zero Flux with
coordinate-wise and full-state one-hot Zero Flow on the binary parity family
\[
\begin{aligned}
    P_0(\mathbf x)
    &=2^{-8}\left[1+\eta(-1)^{\sum_{i=1}^k x_i}\right],
    &
    P_1(\mathbf x)
    &=2^{-8}\left[1-\eta(-1)^{\sum_{i=1}^k x_i}\right],
\end{aligned}
\]
for \(\mathbf x\in\{0,1\}^8\), with \(\eta=0.95\) and
\(k=1,\ldots,5\). The two distributions agree on every marginal of order
below \(k\), while
\[
    \operatorname{TV}(P_0,P_1)=\eta,
    \qquad
    \operatorname{KL}(P_0\|P_1)
    =\eta\log\frac{1+\eta}{1-\eta}.
\]
Thus, the interaction order changes while TV and KL remain fixed.

Dividing each population discrepancy by its value at \(k=2\) gives
\(k2^{3-2k}\) for Zero Flux, \(k2^{1-k}\) for coordinate-wise Zero Flow,
and \(1\) for full-state Zero Flow. Figure~\ref{fig:interaction_order}
shows the unnormalized discrepancies. Zero Flux follows its population
reference through \(k=4\) and shows greater relative variability at
\(k=5\), where its population response is smallest. Coordinate-wise
Zero Flow follows the more gradual decrease, whereas full-state Zero Flow
remains nearly constant but underestimates its population value.
On this family, Zero Flux is sensitive to higher-order differences, with
a stronger relative emphasis on lower-order interactions than
coordinate-wise Zero Flow.

\subsubsection{Implementation Details}

Each endpoint pool contains 500 observations. All methods use the same pools
for each interaction order and seed, with endpoints sampled independently
with replacement. We use seeds 0 through 4.

\paragraph{Architecture and training.}
Zero Flux and coordinate-wise Zero Flow use a Transformer with embedding
width 256, 3 layers, 4 attention heads, feed-forward width 1024, learned
additive positional embeddings, and zero dropout. Zero Flux forms
\(\mathbf Z\) by
independently copying each coordinate from either endpoint with probability
one half. Its two categorical heads minimize the cross-entropy
in~\eqref{eq:implementation-ce}.

Both Zero Flow variants minimize the squared-error objective
\[
    \min_v\;\widehat{\mathbb E}\!
    \left\|v(\mathbf M)
        -\bigl[\phi(\mathbf X_1)-\phi(\mathbf X_0)\bigr]\right\|_2^2,
    \qquad
    \mathbf M=\frac{\phi(\mathbf X_0)+\phi(\mathbf X_1)}2.
\]
For coordinate-wise Zero Flow, \(\phi\) concatenates the two-dimensional
one-hot vectors of the eight coordinates. Its input uses the three possible
coordinate midpoint values, \(0\), \(0.5\), and \(1\), and a linear
head predicts the displacement. Predictions are constrained to zero at
coordinates where the endpoints agree. For full-state Zero Flow,
\(\phi\) is the 256-dimensional one-hot encoding of the entire state.
The model is an MLP with two hidden layers of width 512 and GELU activations;
its prediction is constrained to zero when the endpoint states coincide.
All three methods use AdamW with learning rate \(3\times10^{-4}\), weight
decay \(10^{-4}\), batch size 1024, gradient-norm clipping at 1.0, and
10,000 optimization steps.

\paragraph{Evaluation.}
For each method, interaction order, and seed, we fit an alternative model
on pools from \(P_0\) and \(P_1\), and a null model on two independently
sampled pools from \(P_0\). Each score averages 100,000 midpoint evaluations
from fresh pairings within the corresponding pools. Zero Flux uses the
summed squared normalized flux in~\eqref{eq:implementation-score}; each
Zero Flow variant uses \(\widehat{\mathbb E}\|v(\mathbf M)\|_2^2\).
We report \(\widehat D_{\mathrm{alt}}-\widehat D_{\mathrm{null}}\),
averaged across the five seeds, together with one standard deviation and the
corresponding population discrepancy.

\subsection{Tracing Partisan Distributional Separation over Time in U.S. Nationscape}
\label{app:nation_comp}

We compare Zero Flux with mixture MMD and median MMD on the same
monthly Nationscape distributions. To compare their temporal
patterns, we standardize each trajectory to zero mean and unit
variance across months.
Figure~\ref{fig:nation_comp} shows broadly similar trends in the
earlier months, but the methods diverge in January 2021:
Zero Flux rises sharply, whereas both MMD estimates decline.
The increase captured by Zero Flux is qualitatively more consistent
with the deepening partisan divisions documented around the
2020 election and its aftermath~\citep{sides2023bitter}.

To interpret the Zero Flux trajectory, we examine its policy item
decomposition in Figure~\ref{fig:Nationscape_Heatmap}.
Each contribution uses posterior comparisons conditioned on the
complete midpoint state, providing a breakdown of the joint
discrepancy by policy item.
The border wall contributes most in every month.
Government insurance, deportation, and college also contribute
substantially, while gun background checks and estate tax
contribute comparatively little.
The rise in January 2021 spans several items, including government
insurance, college, and gun restrictions, while the wall
contribution changes little.
This distinguishes the items that account for persistently large
separation from those contributing to changes in the overall
trajectory.

\begin{figure}
    \centering
    \includegraphics[width=0.93\linewidth]{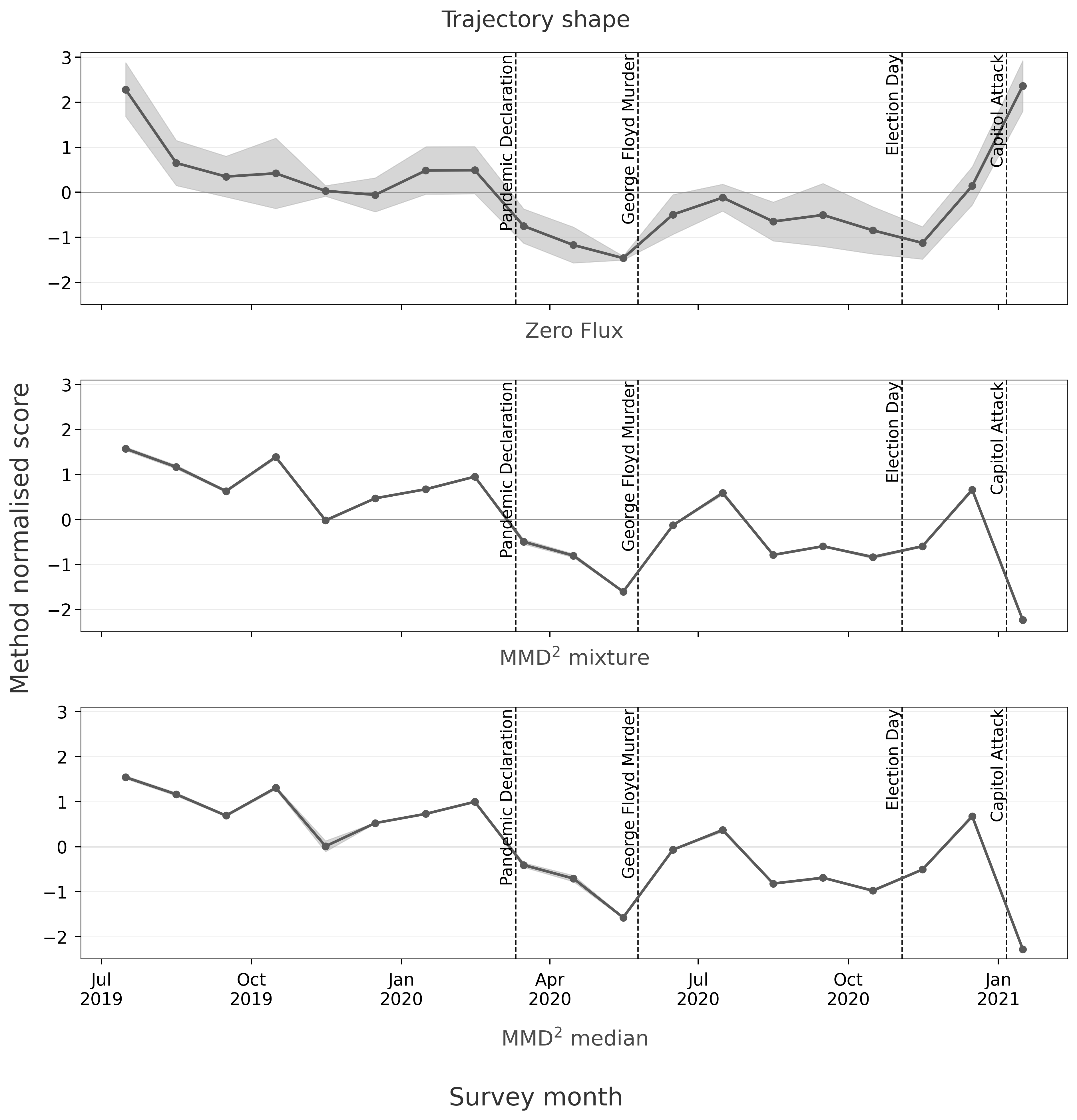}
    \caption{Comparison of monthly Nationscape trajectories for Zero Flux,
mixture MMD, and median MMD.
Each method's trajectory is standardized to zero mean and unit
variance across months.
Curves show means over three seeds; shaded bands show one
standard deviation. Dashed lines mark selected major events.}
    \label{fig:nation_comp}
\end{figure}

\begin{figure}
    \centering
    \includegraphics[width=0.99\linewidth]{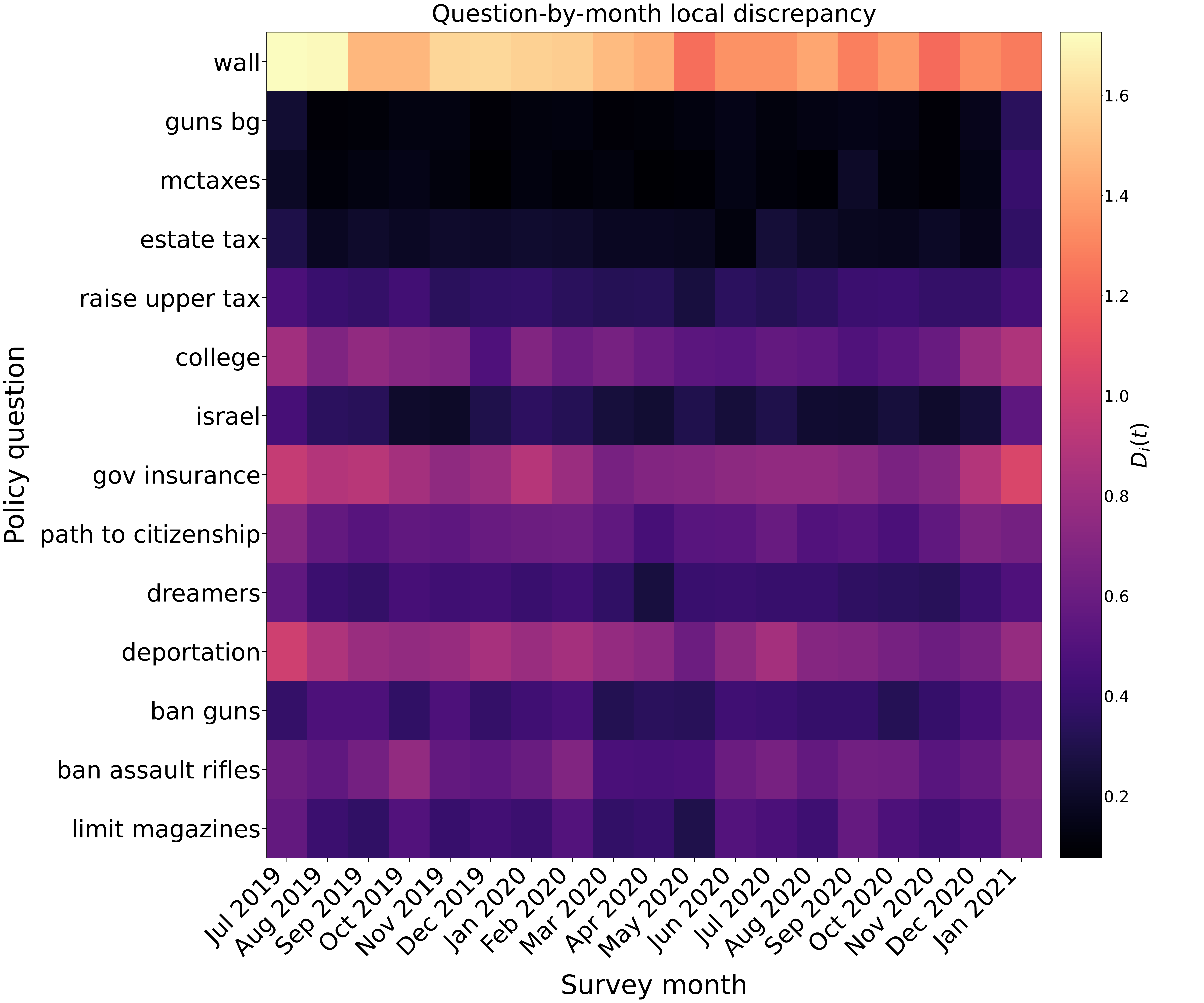}
    \caption{Policy-item decomposition of the monthly Zero Flux discrepancy
between Democrats and Republicans in Nationscape.
Each cell shows the contribution of one item in one month,
averaged over three seeds.
Brighter colors indicate larger contributions on a common scale.}
    \label{fig:Nationscape_Heatmap}
\end{figure}
\end{document}